\documentclass[11pt,a4paper]{article}

\usepackage[T1]{fontenc}
\usepackage[utf8]{inputenc}
\usepackage{lmodern}
\usepackage{newunicodechar}
\newunicodechar{—}{---}
\newunicodechar{ł}{\l{}}
\usepackage[a4paper,margin=1in]{geometry}
\usepackage{amsmath,amssymb,amsthm}
\usepackage{graphicx}
\usepackage{booktabs}
\usepackage{multirow}
\usepackage{subcaption}
\usepackage[numbers,sort&compress]{natbib}
\usepackage{xcolor}
\usepackage{microtype}
\usepackage{hyperref}

\hypersetup{
    colorlinks=true,
    citecolor=blue,
    linkcolor=blue,
    urlcolor=blue,
    pdftitle={DynG-Diff: A State-Aware Dynamic Guidance Diffusion Framework for Probabilistic Time Series Forecasting},
    pdfauthor={Zhente Zhang, Zhengwei Ni, and Wei Fan},
    pdfkeywords={Multivariate time series, probabilistic forecasting, diffusion, dynamic guidance, state-aware network, information heterogeneity}
}

\theoremstyle{definition}
\newtheorem{theorem}{Theorem}
\newtheorem{definition}{Definition}

\title{DynG-Diff: A State-Aware Dynamic Guidance Diffusion Framework for Probabilistic Time Series Forecasting}
\author{%
Zhente Zhang\textsuperscript{1},
Zhengwei Ni\textsuperscript{1,*}, and
Wei Fan\textsuperscript{2}\\[0.6em]
\small \textsuperscript{1}School of Information and Electronic Engineering (Sussex Artificial Intelligence Institute),\\[-0.1em]
\small Zhejiang Gongshang University, Hangzhou, Zhejiang, China\\
\small \textsuperscript{2}School of Computer Science, University of Auckland, Auckland, New Zealand\\[0.4em]
\small \textsuperscript{*}Corresponding author: \href{mailto:zhengwei.ni@zjgsu.edu.cn}{zhengwei.ni@zjgsu.edu.cn}\\
\small \href{mailto:zhangzhente@163.com}{zhangzhente@163.com};
\href{mailto:wei.fan@auckland.ac.nz}{wei.fan@auckland.ac.nz}
}
\date{}

\begin{document}
\maketitle

\begin{abstract}
\noindent Probabilistic multivariate time series (MTS) forecasting is crucial for modeling complex dynamical systems. However, existing diffusion-based methods rely on task-specific conditional paradigms that lack flexibility and struggle with inherent "information heterogeneity"—the significantly varying noise levels and evolutionary patterns across variables. To address this, we propose DynG-Diff, a variable-sensitive dynamic guidance diffusion framework for probabilistic multivariate time-series forecasting: (1) DynG-Diff adopts a two-stage separated training strategy and uses an unconditional diffusion backbone to model the joint distribution of multivariate time series. (2) DynG-Diff introduces a lightweight state-aware policy network that adaptively infers variable reliability from real-time noisy states and one-step denoising estimates, outputting a dynamic guidance strength matrix. (3) DynG-Diff mathematically formulates this dynamic weight as the local precision of the observation distribution, enabling precise guidance for high-confidence variables during inference while filtering out interference from anomalous noise. Extensive experiments on real-world benchmarks demonstrate competitive probabilistic forecasting performance against state-of-the-art conditional diffusion models and improved robustness under severe observation corruption. The implementation code is available at: \url{https://github.com/TT-20011031/DynG-Diff}
\end{abstract}


\vspace{1em}

\section{Introduction}
Probabilistic forecasting of multivariate time series (MTS) is the foundation for modeling complex dynamical systems such as energy grids \cite{nowotarski_recent_2018}, transportation \cite{huang_metaprobformer_2023}, finance \cite{gao_diffsformer_2024}, and healthcare \cite{teng_stocast_2020}. A core challenge in this domain lies in capturing the intricate heterogeneity \cite{nie_time_2023} and coupling \cite{liu_itransformer_2024,zhang_crossformer_2023} within high-dimensional data: different variables often exhibit distinctly different physical characteristics, noise levels, and distribution shifts, yet their temporal evolutions are highly interdependent. Although traditional deterministic methods are effective at capturing trend components, they often fall short in characterizing the inherent stochasticity and multimodal distributions of such systems \cite{salinas_deepar_2020}. This limitation has driven a paradigm shift toward generative probabilistic models \cite{salinas_high-dimensional_2019,li_generative_2022,rasul_multivariate_2021}, which aim to fundamentally learn the underlying joint data distribution. However, effectively modeling these distributions requires not only capturing global dependencies but also respecting the varying information densities of individual variables—a requirement that poses a significant challenge to existing generative frameworks.

Given the limitations of traditional methods, diffusion models \cite{sohl-dickstein_deep_2015,ho_denoising_2020} have rapidly become the framework of choice for generative modeling, owing to their progressive noise removal mechanisms and exceptional distribution-fitting capabilities. Consequently, conditional diffusion models have emerged as a prominent approach. Existing research primarily translates forecasting into a conditional generation problem through two strategies: the first is an end-to-end conditioning strategy, which directly embeds historical observations as conditions into the denoising network to guide generation \cite{tashiro_csdi_2021,alcaraz_diffusion-based_2023,li_channel-aware_2024}; the second is a divide-and-conquer strategy, which utilizes deterministic models to capture trends and then relies on diffusion models to model the residual distribution \cite{li_diffusion-based_2025,lai_rdit_2025}. Although these methods have achieved notable results on specific benchmarks, they are generally constrained by a "task-specific" training paradigm—meaning the model must be specifically trained for a particular prediction horizon or task objective. Once the scenario changes, the model must be retrained, lacking the flexibility expected of generative models. Meanwhile, the few existing studies on unconditional generation primarily focus on the image \cite{dhariwal_diffusion_2021,ho_classifier-free_2022} and audio \cite{kong_diffwave_2021} domains, or are limited to univariate time series \cite{kollovieh_predict_2023}, leading to two critical scientific gaps that urgently need addressing. First, the unconditional generative potential of diffusion models in the MTS domain is severely underestimated. Existing unconditional frameworks have failed to fully demonstrate their capability as "universal generative priors," where a single task-agnostic foundation model can be flexibly adapted to diverse downstream tasks through posterior guidance during inference, without requiring parameter fine-tuning for each task. Second, existing guidance mechanisms struggle to address the "information heterogeneity" challenge in multivariate systems. Unlike univariate sequences, different channels in real-world multivariate systems often exhibit significantly different statistical properties, noise levels, and physical correlation strengths. When using observational data to guide the generation process, the model must possess the ability to distinguish between "primary and secondary" variables and their "signal-to-noise ratios." That is, it must intelligently apply precise guidance to high-confidence key variables while tolerating the randomness of high-noise variables. However, designing a "variable-sensitive" adaptive mechanism to dynamically quantify and utilize this heterogeneous information in the absence of explicit labels remains the key to improving the forecasting accuracy of diffusion models in complex MTS systems.

To overcome these challenges, this paper proposes DynG-Diff (Dynamic Guidance Diffusion), a variable-aware probabilistic forecasting framework for multivariate time series. DynG-Diff maintains the diffusion backbone network in an unconditional generation mode and designs a lightweight State-Aware Policy Network as a plug-in. Rather than relying on manually set hyperparameters, this network adaptively infers the "reliability" of each variable at the current timestep based on the real-time noisy state and one-step estimation during the diffusion process, outputting a Dynamic Guidance Strength Matrix. Mathematically, we interpret this dynamic weight as a plug-in local precision of the observation distribution, yielding a weighted observation-likelihood guidance objective. This design enables the model to intelligently apply precise guidance to high-confidence key variables during inference while automatically reducing the weight interference from highly noisy or anomalous variables. Notably, the policy network is optimized independently while the unconditional backbone remains frozen. This separation allows the guidance mechanism to be adapted without jointly retraining the diffusion backbone.

To address the limitations of existing time series forecasting models in handling complex data distributions and quantifying uncertainty, we propose DynG-Diff, a Dynamic-Guidance Diffusion Model designed specifically for probabilistic time series forecasting. The main contributions of this paper are summarized as follows:

\begin{itemize}
\item We propose DynG-Diff, a decoupled probabilistic forecasting framework that explicitly addresses information heterogeneity by combining an unconditional diffusion backbone with state-dependent, variable-specific guidance during inference. This separation enables the framework to selectively exploit reliable observations without jointly retraining the diffusion backbone.

\item We develop a self-supervised State-Aware Policy Network that estimates variable reliability from the current noisy state and its one-step denoising estimate, producing a dynamic guidance matrix for each variable and diffusion timestep. We further interpret this matrix as a plug-in local precision of an Asymmetric Laplace Distribution and derive a weighted observation-likelihood objective, providing a probabilistic motivation for adaptive guidance.

\item We conduct extensive experiments on real-world benchmark datasets across various domains. The results demonstrate that DynG-Diff achieves highly competitive performance in both forecasting accuracy and probabilistic metrics. Furthermore, comprehensive ablation studies and visualization analyses validate the robustness of the proposed architecture and its substantial potential for practical applications. 
\end{itemize}

The remainder of this paper is organized as follows. Section \ref{sec:related work} reviews related work. Section \ref{sec:preliminaries} covers the problem formulation and preliminaries. Section \ref{sec:method} details our proposed DynG-Diff framework. Section \ref{sec:experiments} evaluates the model through extensive experiments, and Section \ref{sec:Conclusion} concludes the paper.

\section{Related Work}
\label{sec:related work}

\textbf{Time Series Generation.} In recent years, deep generative models have shown immense potential in handling data generation tasks across various domains, and time series generation, being one of the most challenging tasks in the generative field, has also received widespread attention. In early explorations, Generative Adversarial Networks (GANs) played a dominant role. Mogren et al. \cite{mogren_c-rnn-gan_2016} pioneered C-RNN-GAN, innovatively integrating Long Short-Term Memory (LSTM) networks into the generator and discriminator of GANs to process continuous sequential data. Esteban et al. \cite{esteban_real-valued_2017} introduced a conditioning mechanism and developed the Recurrent Conditional Generative Adversarial Network (RCGAN), successfully achieving the synthesis of multidimensional real-valued time series using auxiliary label information. Yoon et al. \cite{yoon_time-series_2019} proposed TimeGAN, a framework that explicitly constrains the temporal dynamics of data within a jointly optimized latent space, thereby generating higher-fidelity samples. Paul et al. \cite{paul_psa-gan_2021} proposed PSA-GAN, which significantly improves the synthesis quality of long multivariate time series by incorporating a progressive growing strategy and self-attention mechanisms.

Due to the inherent instability of adversarial training, researchers began to explore other types of deep generative paradigms. For instance, TimeVAE, proposed by Desai et al. \cite{desai_timevae_2021}, implemented an interpretable temporal structure and achieved initial success in using VAEs for multivariate time series synthesis. HyVAE, proposed by Cai et al. \cite{cai_hybrid_2023}, integrated the joint learning of local patterns (e.g., seasonality and trend) and temporal dynamics of time series into a unified framework via variational inference. Wu et al. \cite{wu_k_2025} combined Koopman theory and Kalman filtering, utilizing KoopmanNet to transform nonlinear time series dynamics into a biased linear dynamical system, and employing KalmanNet to refine predictions and model uncertainties within this linear system. Additionally, energy-based models have been used to mimic the sequential behavior of time series through progressively decomposable structures. Alaa et al. \cite{alaa_generative_2021} proposed Fourier Flows based on normalizing flows and a series of spectral filters to achieve exact likelihood optimization.

\textbf{Time Series Diffusion Models.} Denoising Diffusion Probabilistic Models (DDPMs), as a novel class of generative models, have been widely applied to time series forecasting and imputation tasks. The pioneering work in this area was introduced by Rasul et al. \cite{rasul_autoregressive_2021}, whose designed TimeGrad combines autoregressive models with the denoising diffusion process to achieve multivariate probabilistic forecasting. Subsequently, CSDI, proposed by Tashiro et al. \cite{tashiro_csdi_2021}, constructed a score-based conditional diffusion framework that unified time series imputation and forecasting tasks through self-attention and a specialized masking strategy. As research deepened, to overcome the computational bottleneck of long sequence modeling, SSSD \cite{alcaraz_diffusion-based_2023} replaced the attention mechanism in traditional diffusion architectures with structured state space models (S4). TimeDiff, proposed by Shen et al. \cite{shen_multi-resolution_2024}, introduced two novel conditioning mechanisms: future mixup and autoregressive initialization, which significantly improved the forecasting quality of long sequences while effectively capturing complex temporal dynamics. For example, Yuan and Qiao \cite{yuan_diffusion-ts_2024} achieved highly interpretable time series generation by decoupling trend and seasonal priors. However, these methods generally couple joint-distribution learning with task-specific conditioning. In contrast, DynG-Diff combines unconditional pre-training with variable-aware inference guidance to address information heterogeneity in probabilistic multivariate time-series forecasting.

\textbf{Diffusion Guidance.} Classifier Guidance adds the gradient of an auxiliary classifier to the unconditional score during sampling \cite{dhariwal_diffusion_2021}. Classifier-Free Guidance (CFG) removes the auxiliary classifier by jointly learning conditional and unconditional score estimates and blending them with a user-defined guidance scale \cite{ho_classifier-free_2022}. Guidance has also been used for text-driven image generation and editing \cite{avrahami_blended_2022,nichol_glide_2021}, while constraint-based guidance expresses desired time-series properties as differentiable energy functions \cite{coletta_constrained_2023}. For time series, TSDiff conditions an unconditionally trained diffusion model through observation self-guidance during inference, without an auxiliary guidance network or changes to backbone training \cite{kollovieh_predict_2023}. Its observation-likelihood gradient is regulated by a globally shared scale. Feedback Guidance makes the guidance coefficient state- and time-dependent by feeding back the model's estimate of conditional-signal informativeness \cite{koulischer_feedback_2025}. However, it adapts guidance at the sample-trajectory level rather than estimating variable-wise observation reliability.

DynG-Diff is rooted in TSDiff's observation-likelihood formulation: it retains the unconditional backbone, the one-step estimate $\hat{x}^0$, and the inference-time likelihood gradient. It extends this formulation with a separately trained policy network, a self-supervised reliability target, and a variable- and timestep-specific matrix $A^t$. Unlike CFG and Feedback Guidance, DynG-Diff does not interpolate conditional and unconditional score estimates; unlike TSDiff, it does not apply one shared guidance strength to all observed variables.

\section{Preliminaries}
\label{sec:preliminaries}

\subsection{Problem statement}
\label{subsec:3.1}

This paper investigates the problem of probabilistic multivariate time series forecasting. Given a multivariate time series $y \in \mathbb{R}^{L \times D}$ of length $L$ containing $D$ variables. For the forecasting task, we introduce a binary observation mask $M \in \{0,1\}^{L \times D}$ with the same dimensions as $y$. The mask $M$ indicates the observation status of the data points: $M_{l,d} = 1$ denotes that the data at this position is a known historical observation, denoted as $y_{obs}$; whereas $M_{l,d} = 0$ indicates that the data is an unknown future value to be predicted, denoted as $y_{target}$. The known observations yobs and the unknown targets ytarget can be expressed as:
\begin{equation}
y_{obs} = M \odot y, \quad y_{target} = (1 - M) \odot y
\label{eq:1}
\end{equation}
where $\odot$ denotes the element-wise product. Unlike deterministic forecasting, which solely aims to find a mapping $f: y_{obs} \rightarrow \hat{y}_{target}$ to minimize the point error, the goal of probabilistic forecasting is to learn the conditional probability distribution of the target values given the observations, $p_{\theta}(y_{target}|y_{obs})$. During practical inference, the model approximates this distribution by generating a set of samples $\{\hat{y}^{s}\}_{s=1}^{S} \sim p_{\theta}(\cdot|y_{obs})$. Consequently, it not only provides the predictive mean but also quantifies the predictive uncertainty through the statistical properties of the samples, effectively capturing the stochasticity and coupling relationships of the multivariate system as it evolves over time.

\subsection{Denoising Diffusion Probabilistic Models}
\label{subsec:3.2}

Diffusion models \cite{sohl-dickstein_deep_2015,ho_denoising_2020} are a class of generative models based on non-equilibrium thermodynamics; their core idea encompasses two processes: forward noise addition and reverse denoising. The normalized multivariate time series $y$ is treated as the noise-free initial state of the diffusion process, explicitly denoted as $x^{0} \in \mathbb{R}^{L \times D}$. During the diffusion process, $x^{t} \in \mathbb{R}^{L \times D}$ is defined as the latent state variable at an arbitrary diffusion time step $t \in [1,T]$. The forward process is a fixed Markov chain, where the single-step transition probability is parameterized as $q(x^{t}|x^{t-1}) = \mathcal{N}(x^{t}; \sqrt{1 - \beta_{t}}x^{t-1}, \beta_{t}I)$, and $\beta_{t} \in (0,1)$ is the variance schedule parameter. Utilizing the reparameterization trick, the latent variable at any arbitrary timestep t can be sampled directly from the initial data $x^{0}$: $x^{t} = \sqrt{\bar{\alpha}_{t}}x^{0} + \sqrt{1 - \bar{\alpha}_{t}}\epsilon$, where $\alpha_{t} = 1 - \beta_{t}$, $\bar{\alpha}_{t} = \prod_{i=1}^{t} \alpha_{i}$, and $\epsilon \sim \mathcal{N}(0, I)$ is standard Gaussian noise.

The reverse process aims to learn the conditional probability distribution to reverse the noise addition process. Since the true reverse transition distribution is intractable, Ho et al. \cite{ho_denoising_2020} approximate the noise that needs to be removed at each step by training a neural network $\epsilon_{\theta}(x^{t}, t)$, whose training objective employs a simplified mean squared error (MSE) loss function:
\begin{equation}
\mathcal{L}_{simple} = \mathbb{E}_{x^{0}, \epsilon, t}[\|\epsilon - \epsilon_{\theta}(x^{t}, t)\|^{2}]
\label{eq:2}
\end{equation}

Notably, based on the noise $\epsilon_{\theta}(x^{t}, t)$ predicted by the model and the reparameterized $x^{t}$, we can derive the estimation of the original data at the current step:
\begin{equation}
\hat{x}^{0} = \frac{x^{t} - \sqrt{1 - \bar{\alpha}_{t}}\epsilon_{\theta}(x^{t}, t)}{\sqrt{\bar{\alpha}_{t}}}
\label{eq:3}
\end{equation}

This capability of real-time estimation of the original data $x^{0}$ during the generation process serves as the critical foundation for realizing observation-based dynamic guidance mechanisms. For more details, please refer to Appendix \ref{app:A}.

\subsection{Observation Guidance}
\label{subsec:3.3}

Unconditional diffusion models learn the joint distribution of the data $p(x)$ . In time series forecasting tasks, it is necessary to sample from the conditional distribution $p(x|y_{obs})$ \cite{kollovieh_predict_2023}. According to Bayes' theorem, the conditional score function can be decomposed into the sum of an unconditional score term and an observation likelihood guidance term \cite{dhariwal_diffusion_2021,kollovieh_predict_2023}:
\begin{equation}
\nabla_{x^t} \log p (x^t|y_{obs}) = \underbrace{\nabla_{x^t} \log p (x^t)}_{\text{Uncond. Score}} + \underbrace{\nabla_{x^t} \log p (y_{obs}|x^t)}_{\text{Observation Guidance}} \quad
\label{eq:4}
\end{equation}
Here, the first term is provided by the pre-trained denoising network $\epsilon_\theta(x^t, t)$, while the second term measures the consistency between the currently generated latent variable $x^t$ and the historical observations $y_{obs}$. Since directly computing $p(y_{obs}|x^t)$ is intractable, we can utilize the original data estimation $\hat{x}^0$ obtained from Eq. \eqref{eq:3} to approximate the true $x^0$ , thereby calculating the observation guidance loss $\mathcal{L}_{guide}$ . We then compute its gradient and inject it into the predicted noise. The modified noise prediction can be expressed as:
\begin{equation}
\hat{\epsilon} = \epsilon_\theta(x^t, t) - s \cdot \sqrt{1 - \bar{\alpha}_t} \cdot \nabla_{x^t} \mathcal{L}_{guide}(\hat{x}^0, y_{obs}) \quad
\label{eq:5}
\end{equation}
Here, $\mathcal{L}_{guide}$ depends on the assumed observation noise distribution, and $s$ is a global guidance scale shared by all observed variables. Equation \eqref{eq:5} therefore represents homogeneous observation self-guidance and serves as the starting point of DynG-Diff. Our method retains this posterior-score formulation but replaces homogeneous loss weighting with state-aware local precision.

\section{DynG-Diff: Variable-Sensitive Dynamic Guidance Diffusion}
\label{sec:method}

In this section, we present DynG-Diff, a novel probabilistic forecasting framework for multivariate time series, which aims to address the inherent information heterogeneity challenge in complex dynamical systems by combining an unconditional diffusion backbone with variable-sensitive adaptive guidance. As illustrated in Figure \ref{FIG:1}, the overall architecture of DynG-Diff is designed into three stages: the unconditional backbone network pre-training stage, the state-aware policy learning stage, and the dynamic guidance inference stage. This decoupled paradigm preserves the backbone's capability to fit joint distributions while enabling fine-grained interventions for different variables during inference.

\textbf{Relationship to Existing Guidance.} DynG-Diff is not a variant of CFG because it neither trains a conditional denoising branch nor blends conditional and unconditional score estimates. Its direct methodological starting point is the homogeneous observation self-guidance in Eq. \eqref{eq:5}.

Beyond this starting point, DynG-Diff adds four components: the State-Aware Policy Network $g_\phi$, a self-supervised reliability target derived from denoising error, the variable- and timestep-specific guidance matrix $A^t$, and an ALD-based weighted observation-likelihood objective. These additions transform a shared scalar into fine-grained local precision while preserving the unconditional backbone.

In the subsequent sections, we detail the core components and theoretical derivations of the DynG-Diff framework. First, serving as the generative foundation of the entire framework, the unconditional backbone network is independently optimized during the pre-training stage using a standard denoising diffusion objective to fit the underlying joint probability distribution of the multivariate time series. Once pre-training is complete, the frozen backbone network not only provides a powerful unconditional generative prior for subsequent inference sampling, but its real-time one-step denoising estimation output during intermediate diffusion steps also serves as the core data source for the subsequent policy network to quantify the system's recovery state and variable reliability. For more details on the pre-training of the unconditional backbone network, please refer to Appendix \ref{subapp:B.1}. Subsequently, in Section \ref{subsec:4.1}, we elaborate on the design motivation and specific architecture of the state-aware policy network, explaining how it utilizes the real-time noisy state and one-step denoising estimation during the diffusion process to quantify the reliability of different variables and generate the dynamic weight matrix. Finally, in Section \ref{subsec:4.2}, we introduce the variable-sensitive guidance mechanism, proving from a probabilistic perspective how dynamic weights influence the model's tolerance to observation errors, and we derive the guidance loss based on the weighted observation likelihood in detail.

\begin{figure}[!t]
	\centering
	\includegraphics[width=.86\linewidth]{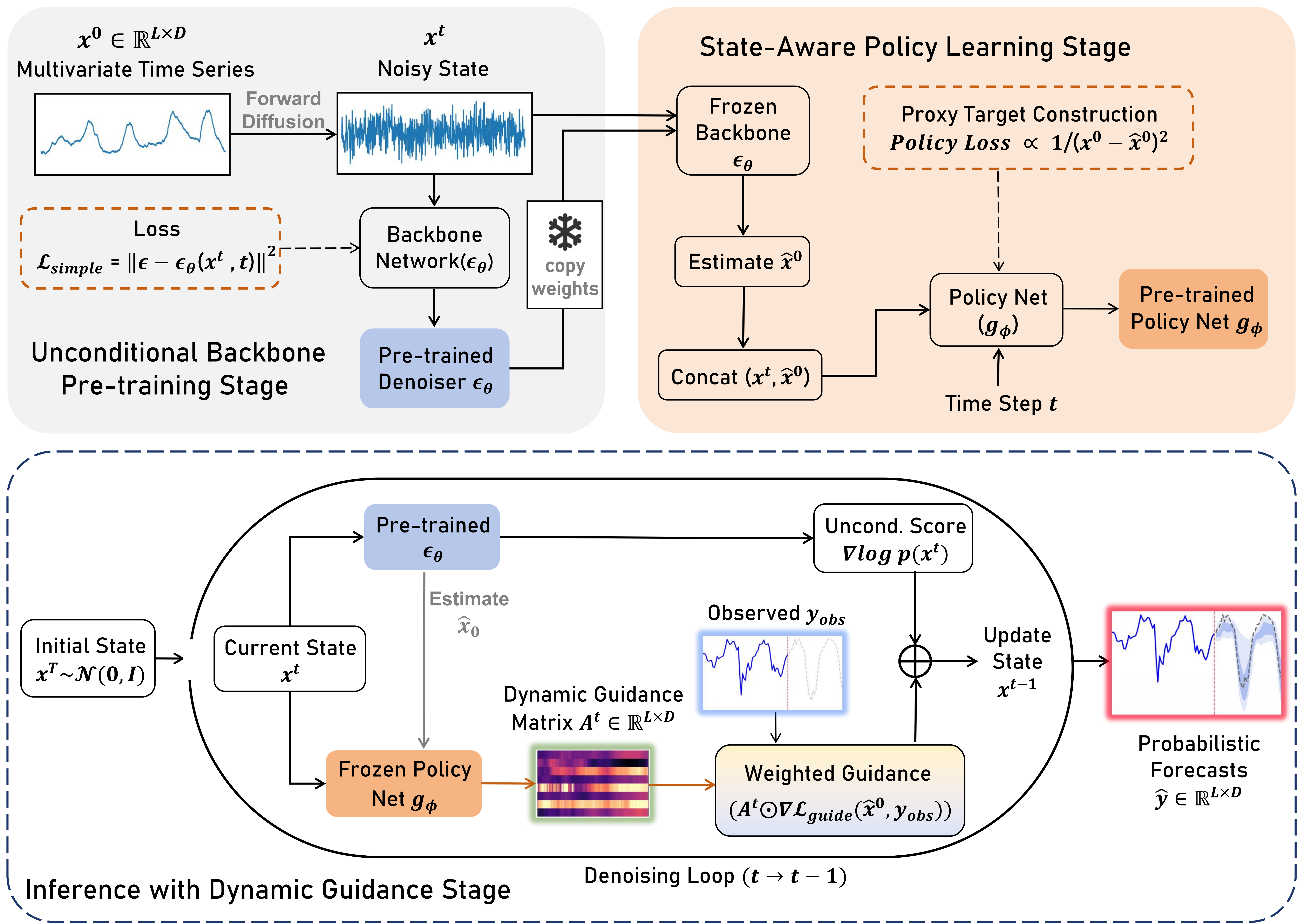}
	\caption{Overall architecture of DynG-Diff. The framework comprises three core stages: (1) Unconditional Backbone Pre-training Stage: the backbone denoising network is trained using the standard diffusion loss $\mathcal{L}_{simple}$ to capture the underlying joint probability distribution of the multivariate time series $x^0 \in \mathbb{R}^{L \times D}$. (2) State-Aware Policy Learning Stage: the pre-trained denoising network $\epsilon_\theta$ is frozen. Its one-step estimate $\hat{x}^0$ is concatenated with the current noisy state $x^t$ and fed into the policy network $g_\phi$. Using the reconstructed inverse error as a proxy target, the policy network learns and outputs a dynamic guidance strength matrix $A^t \in \mathbb{R}^{L \times D}$ tailored to each diffusion timestep and variable. (3) Inference with Dynamic Guidance Stage: at each reverse diffusion step, the frozen policy network $g_\phi$ recomputes $A^t$. Its detached copy $\bar{A}^t=\operatorname{sg}(A^t)$ modulates the observation-guidance gradient in a variable-wise manner.}
	\label{FIG:1}
\end{figure}

\subsection{State-Aware Dynamic Weight Generation}
\label{subsec:4.1}

Real-world complex multivariate systems often exhibit significant information heterogeneity, with different variables frequently displaying vastly different dynamic evolution patterns. Based on this, we propose a state-aware dynamic weight generation mechanism aimed at adaptively allocating fine-grained guidance weights for different variables.

\textbf{Policy Network Design.} To endow the model with this dynamic perception capability, we introduce an auxiliary lightweight policy network alongside the unconditional diffusion backbone. Its overall architecture and learning process are illustrated in Figure \ref{FIG:2}. The core objective of this network is to accurately represent the recovery state of the system at the current diffusion timestep and subsequently infer the relative reliability of each variable. Since the true noise-free historical data $x^0$ is unobservable during the inference stage, the model must quantify the current generation quality based solely on intermediate states. Given the progressive denoising nature of diffusion models, the current noisy state $x^t$ and the one-step denoising estimation $\hat{x}^0(x^t, t)$ output by the backbone network (derived from Eq. \eqref{eq:3}) contain rich local evolution information. Therefore, we transpose both tensors and concatenate them along the channel dimension to construct the composite state representation of the current system, $\mathcal{S}^t = [(x^t)^T \parallel (\hat{x}^0)^T] \in \mathbb{R}^{2D \times L}$. This state representation strategy not only preserves the current true noise distribution of the observation sequence but also fully integrates the unconditional prior knowledge already learned by the backbone model. Meanwhile, to equip the policy network with time-awareness, we map the discrete diffusion timestep $t$ into a continuous high-dimensional vector via sinusoidal positional encoding, and extract the time embedding feature $e^t \in \mathbb{R}^{d_t}$ through a Multi-Layer Perceptron (MLP).

Regarding the network architecture design, considering the stringent requirements of time series forecasting on inference speed and computational overhead, the policy network employs a series of lightweight one-dimensional convolutional (1D-CNN) layers as the core feature extractor. The input state $\mathcal{S}^t$ first passes through non-linear convolutional layers to extract cross-variable local correlations and temporal features, generating a high-dimensional hidden state representation. Subsequently, the time embedding $e^t$ is deeply fused with the hidden state along the feature dimension and fed into a Weight Prediction Head composed of activation functions and 1D convolutions. Assuming the network's hidden variable output before the final layer mapping is $H^t \in \mathbb{R}^{D \times L}$ (where $D$ is the variable dimension and $L$ is the sequence length), to ensure that the generated weights reasonably represent the relative "state confidence" among variables and to prevent the introduction of dynamic weights from shifting the global gradient scale, we adopt a weight generation strategy based on exponential mapping and mean normalization. Specifically, the hidden variable is first mapped to a relative score matrix via a 1D convolutional layer:
\begin{equation}
Z^t = W_{head} \ast H^t + b_{head} \quad
\label{eq:6}
\end{equation}
where $W_{head}$ and $b_{head}$ are the weights and biases of the prediction head's convolutional layer, respectively. Subsequently, an exponential operation is applied to ensure the non-negativity of the weights, followed by normalization by dividing by the global mean:
\begin{equation}
A^t = \frac{\exp (Z^t)}{\frac{1}{D \cdot L} \sum_{i=1}^{D} \sum_{j=1}^{L} \exp ((Z^t)_{i,j}) + \epsilon} \quad
\label{eq:7}
\end{equation}
where $\epsilon$ is a tiny constant to prevent division by zero, and $A^t$ is the final generated non-negative guidance strength matrix, with a global mean approximately equal to 1. For its probabilistic interpretation, we introduce the following modeling definition:

\begin{definition}[Dynamic Local Precision]
Any element $a_{i,j}^{t}$ in the dynamic guidance strength matrix $A^{t}$ is interpreted as a plug-in estimate of the reciprocal scale of an Asymmetric Laplace Distribution (ALD), i.e., the local precision. It represents the model's inverse error tolerance for variable $i$ at time point $j$ and diffusion timestep $t$.
\label{def:1}
\end{definition}

Based on Definition \ref{def:1}, the generated weight matrix $A^t$ serves not merely as an attention mask for feature selection, but possesses explicit physical and probabilistic interpretability.

\textbf{Proxy Target Construction.} To enable the policy network to accurately learn the latent local precision described in Definition \ref{def:1}, we design an uncertainty-aware self-supervised training objective. Since the absolute "true reliability" labels for each variable at every intermediate timestep are inaccessible during the diffusion process, we transform the residual between the backbone model's one-step denoising estimation $\hat{x}^0(x^t, t)$ and the true noise-free signal $x^0$ into local precision, constructing a proxy supervision signal based on this.

Specifically, we first compute the squared error between the denoising estimation and the true signal, mapping it to log-precision:
\begin{equation}
P_{raw}^t = -\log ((\hat{x}^0(x^t, t) - x^0)^{\odot 2} + \epsilon) \quad
\label{eq:8}
\end{equation}
where $\epsilon$ is a tiny constant to prevent numerical underflow. Because the absolute magnitude of errors spans vastly across different diffusion stages, directly fitting the absolute precision would cause severe gradient instability in the network. Therefore, we apply Z-Score normalization to $P_{raw}^t$ across spatial and variable dimensions, followed by threshold clipping, to construct a smooth relative confidence target $Y_{target}^t$:
\begin{equation}
Y_{target}^t = Clip\left( \frac{P_{raw}^t - \mu_P}{\sigma_P}, -3, 3 \right) \quad
\label{eq:9}
\end{equation}
where $\mu_P$ and $\sigma_P$ are the mean and standard deviation of the log-precision within the current batch, respectively. Finally, the policy network $g_\phi$ fits this standardized relative confidence target using MSE as the loss function:
\begin{equation}
\mathcal{L}_{policy} = \mathbb{E}_{t, x^0, \epsilon} \left[ \frac{1}{D \cdot L} \sum_{i=1}^{D} \sum_{j=1}^{L} \| (Z^t)_{i,j} - (Y_{target}^t)_{i,j} \|_2^2 \right] \quad
\label{eq:10}
\end{equation}

Under this optimization objective, the policy network's output $Z^t$ represents the relative log-precision of the recovery state for each variable at the current timestep, which is then processed through Eq. \eqref{eq:7} to generate the non-negative dynamic guidance strength matrix $A^t$. The generated guidance strength matrix $A^t$ possesses clear physical significance and probabilistic interpretability. Higher weights in the matrix indicate that the current generation state of the corresponding variable is highly reliable and exhibits clear trend features; thus, strong gradient guidance should be applied in subsequent sampling to force it into strict alignment with the conditional distribution. Conversely, lower weight values imply that the current variable is still in a noise-dominated state with high uncertainty. The network will grant it greater tolerance, allowing the diffusion backbone model to freely explore relying on the unconditional generative prior, thereby avoiding adversarial gradients caused by enforcing strong guidance on noisy estimations. After obtaining this crucial dynamic weight matrix $A^t$, subsequent sections will detail how to seamlessly integrate it into the Bayes' theorem-based diffusion sampling to drive differentiated gradient updates.

\begin{figure}[!t]
	\centering
	\includegraphics[width=.86\linewidth]{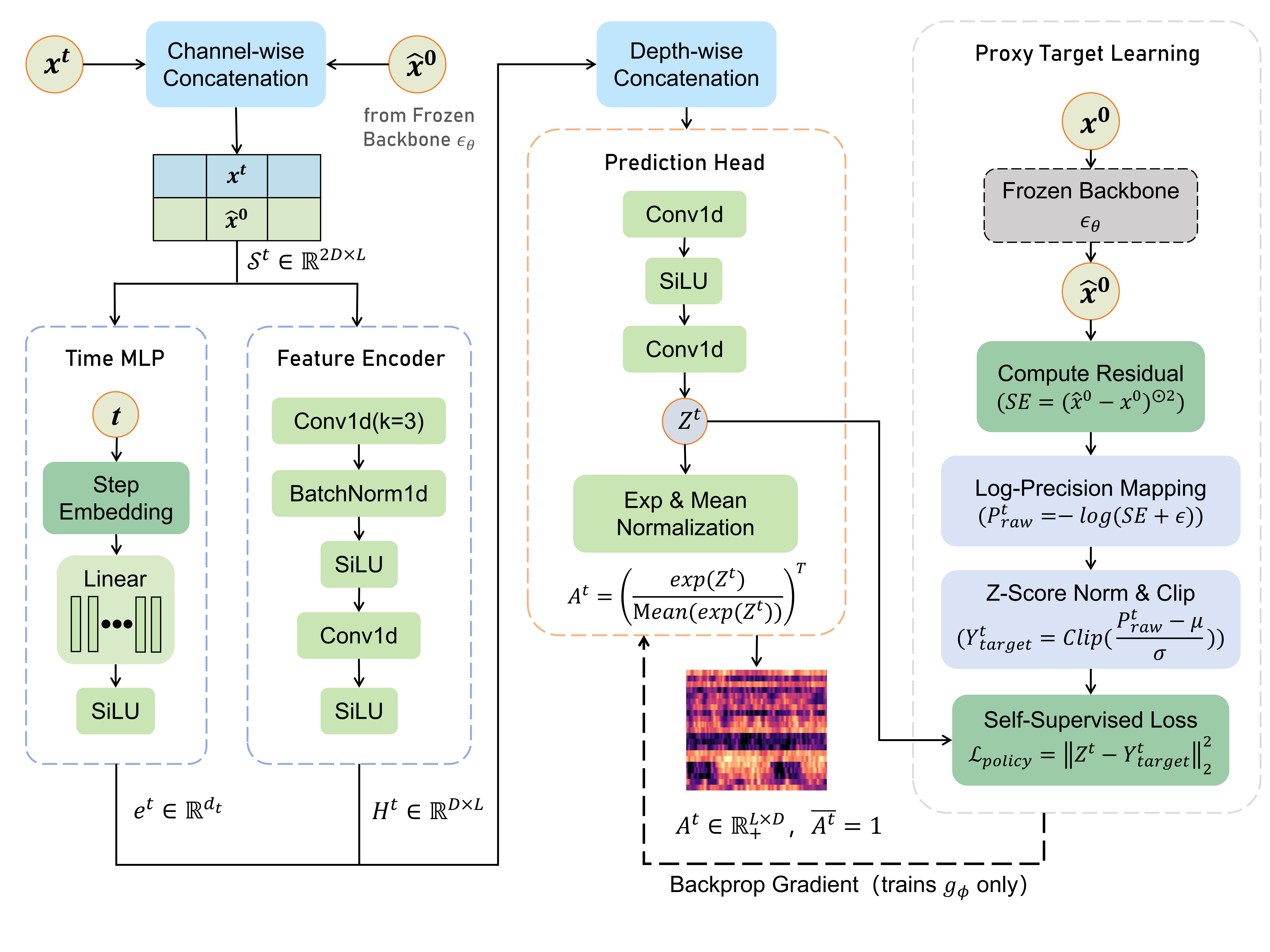}
	\caption{Schematic of the State-Aware Policy Network architecture. The network takes the concatenated features of the current noisy state $x^t$ and the one-step estimate $\hat{x}^0$ as input, combines them with the timestep embedding, and outputs the dynamic weight matrix $A^t$. During training, the diffusion backbone remains frozen. The model computes the log-precision from the squared error between $\hat{x}^0$ and the ground truth $x^0$ and, after Z-score normalization, constructs a smooth proxy target $Y_{target}^t$ to update the policy network parameters through the loss $\mathcal{L}_{policy}$.}
	\label{FIG:2}
\end{figure}

\subsection{Variable-Sensitive Guidance Mechanism}
\label{subsec:4.2}

Recent advances in the field of time series generation \cite{kollovieh_predict_2023} indicate that by introducing observation likelihood guidance during inference, unconditional diffusion models can flexibly adapt to various downstream forecasting and imputation tasks without explicit conditional training. However, existing self-guiding diffusion frameworks are either limited to modeling univariate time series or, when directly generalized to multivariate dynamical systems, ignore the distinctly different dynamic evolution patterns and noise levels inherent in different variables at the same timestep. If gradient guidance of equal strength is applied to all channels, the model is highly prone to producing adversarial gradients on variables containing extreme anomalies or high-noise disturbances, thereby severely degrading the generation quality of the overall multivariate joint distribution. To this end, we propose a variable-sensitive self-guidance mechanism. This mechanism integrates the dynamic guidance strength matrix $A^t$ output by the policy network into the observation-likelihood gradient, providing a local probabilistic interpretation for fine-grained sampling in multivariate systems.

\textbf{Conditional Probability Modeling.} In probabilistic time series forecasting tasks, the model is required not only to provide point estimates but also to accurately delineate predictive uncertainty. Considering that widely used evaluation metrics in this domain (such as the Continuous Ranked Probability Score, CRPS) heavily rely on the quantile loss, we adopt the ALD to model the conditional observation distribution $p(y_{obs}|x^t)$ given the latent variable $x^t$. For theoretical details of ALD modeling, please refer to Appendix \ref{subapp:B.2}.

At each reverse diffusion step, the policy network recomputes $A^t$ from the current state. For likelihood-gradient evaluation, we use $\bar{A}^t=\operatorname{sg}(A^t)$, where $\operatorname{sg}(\cdot)$ denotes the stop-gradient operator. Thus, $\nabla_{x^t}\bar{A}^t=0$. The matrix remains dynamic across reverse steps but is treated as fixed only during differentiation at the current step.

For a single data point $(l, d)$ in a multivariate sequence, the one-step estimation of the original data provided by the backbone network at the current timestep $t$ can be obtained from Eq. \eqref{eq:3}. Under this estimation, the conditional probability density function of the true observation $y_{l,d}$ can be defined as:
\begin{equation}
p(y_{l,d}|x^t; \kappa, \bar{a}_{l,d}^t) \propto \exp(- \bar{a}_{l,d}^t \cdot \rho_\kappa(y_{l,d} - \hat{x}_{l,d}^0)) \quad
\label{eq:11}
\end{equation}
where $\kappa \in (0,1)$ is the specified quantile level, $\rho_\kappa(e) = \max(\kappa \cdot e, (\kappa - 1) \cdot e)$ is the asymmetric quantile loss function, and $\bar{a}_{l,d}^t$ is the corresponding stop-gradient element of $\bar{A}^t$. We interpret $\bar{a}_{l,d}^t$ as a plug-in local precision, i.e., the reciprocal of the scale parameter of the Laplace distribution. A large weight produces a sharper local likelihood and stronger alignment with the observation. A small weight produces a flatter likelihood, allowing uncertain variables to rely more on the unconditional prior.

During inference, trajectory $s$ uses $\kappa_s=s/(S+1)$, $s=1,\ldots,S$. The $S$ trajectories form an equally weighted $\kappa$-conditioned ensemble rather than exact conditional quantiles.

\textbf{Joint Conditional Probability and Guidance Loss.} To achieve global guidance in multivariate systems, we assume that given the local estimations, the individual observed variables are conditionally independent. Thus, the joint conditional probability density of the multivariate time series can be expressed as the product of the individual univariate densities:
\begin{equation}
p(y_{obs}|x^t;\bar{A}^t) = \prod_{(l,d):M_{l,d}=1} p(y_{l,d}|x^t; \kappa, \bar{a}_{l,d}^t) \quad
\label{eq:12}
\end{equation}
where $M \in \{0,1\}^{L \times D}$ is the binary observation mask defined in Section \ref{subsec:3.1}. The conditional independence assumption does not remove cross-variable information from the guidance weights. Before detachment, $A^t$ is inferred through cross-channel feature fusion with a global receptive field, so each plug-in precision can encode system-wide coupling and relative reliability.

Taking the negative logarithm of the joint conditional probability density converts the product into a summation. Discarding constant terms irrelevant to $x^t$, we derive the core energy function of the guided diffusion process—the Weighted Quantile Guidance Loss (WQ-Loss):
\begin{align}
\begin{split}
-\log p(y_{obs}|x^t;\bar{A}^t) &= \mathcal{L}_{guide}(\hat{x}^0,y_{obs};\bar{A}^t)+C(\bar{A}^t,\kappa), \\
\mathcal{L}_{guide}(\hat{x}^0,y_{obs};\bar{A}^t) &= \sum_{l=1}^{L} \sum_{d=1}^{D} M_{l,d} \cdot \bar{a}_{l,d}^t \cdot \rho_\kappa(y_{l,d} - \hat{x}_{l,d}^0) \quad
\end{split}
\label{eq:13}
\end{align}

Here, $C(\bar{A}^t,\kappa)$ collects terms constant with respect to $x^t$ during the current update. The matrix $\bar{A}^t$ acts as an adaptive weighting mask, and Eq. \eqref{eq:5} uses the resulting local likelihood gradient. The following theorem gives this gradient under the stated stop-gradient convention.

\begin{theorem}[Variable-Sensitive Guidance Gradient]
Given $y_{obs}$ and $M\in\{0,1\}^{L\times D}$, assume the local observation probability follows an ALD with quantile level $\kappa\in(0,1)$. For the fixed current-step matrix $\bar{A}^t$, satisfying $\nabla_{x^t}\bar{A}^t=0$, the gradient of the Weighted Quantile Guidance Loss is given by:
\begin{align}
\begin{split}
\nabla_{x^{t}}\mathcal{L}_{guide}(\hat{x}^{0}, y_{obs};\bar{A}^t) &= \sum_{l=1}^{L} \sum_{d=1}^{D} M_{l,d} \cdot \bar{a}_{l,d}^{t} \\
&\quad \cdot (I(y_{l,d} < \hat{x}_{l,d}^{0}) - \kappa) \cdot \nabla_{x^{t}}\hat{x}_{l,d}^{0}
\end{split}
\label{eq:14}
\end{align}
where $I(\cdot)$ denotes the indicator function.
\label{thm:1}
\end{theorem}

\begin{proof}
Based on the definition of the weighted quantile guidance loss in Eq. \eqref{eq:13}, the objective function is essentially the weighted sum of asymmetric quantile losses $\rho_\kappa(e)$ for each data point. For a single data point $(l, d)$, let the residual error be $e = y_{l,d} - \hat{x}_{l,d}^0$.

First, the subgradient of the asymmetric quantile loss function $\rho_\kappa(e)$ with respect to the residual $e$ can be equivalently expressed using the indicator function $I(\cdot)$ as:
\begin{equation}
\frac{\partial \rho_\kappa(e)}{\partial e} = \kappa - I(e < 0) \quad
\label{eq:15}
\end{equation}

Second, according to the backpropagation mechanism, we apply the chain rule to compute the gradient of this local loss with respect to the latent state $x^t$:
\begin{equation}
\nabla_{x^t} \rho_\kappa(y_{l,d} - \hat{x}_{l,d}^0) = \frac{\partial \rho_\kappa(e)}{\partial e} \cdot \nabla_{x^t}(y_{l,d} - \hat{x}_{l,d}^0) \quad
\label{eq:16}
\end{equation}

Since the true observation $y_{l,d}$ is a given constant, its gradient with respect to $x^t$ is $\nabla_{x^t}y_{l,d} = 0$. Substituting this yields:
\begin{align}
\begin{split}
\nabla_{x^t}\rho_\kappa(y_{l,d} - \hat{x}_{l,d}^0) &= (\kappa - I(y_{l,d} - \hat{x}_{l,d}^0 < 0)) \cdot (-\nabla_{x^t}\hat{x}_{l,d}^0) \\&= (I(y_{l,d} < \hat{x}_{l,d}^0) - \kappa) \cdot \nabla_{x^t}\hat{x}_{l,d}^0 \quad
\end{split}
\label{eq:17}
\end{align}

Finally, substituting this expression into Eq. \eqref{eq:13} and summing over positions and variables, with the fixed current-step weights $\bar{a}_{l,d}^t$, completes the proof.
\end{proof}

As derived in Theorem \ref{thm:1}, the quantile error signal $(I(\cdot)-\kappa)$ determines the direction and quantile bias of the gradient. The fixed current-step weight $\bar{a}_{l,d}^t$ scales this signal, while $\nabla_{x^t}\hat{x}_{l,d}^0$ maps it back to the latent sampling direction. Substitution into Eq. \eqref{eq:5} yields variable-sensitive guidance without jointly retraining the unconditional backbone.

\noindent\textbf{Relation to Homogeneous Guidance.} When $\bar{A}^t=\mathbf{1}$, Eq. \eqref{eq:13} reduces to a homogeneous quantile observation loss. Substitution into Eq. \eqref{eq:5} recovers observation self-guidance with one globally shared scale $s$. Thus, homogeneous observation guidance is a special case of DynG-Diff, whereas the recomputed $A^t$ generalizes it to variable-, position-, and diffusion-timestep-specific regulation.

\section{Experiments}
\label{sec:experiments}

\subsection{Experimental setup}
\label{subsec:5.1}

\textbf{Datasets.} We selected six widely used, publicly available real-world multivariate time series benchmark datasets covering complex dynamical systems across diverse domains, including energy, economics, meteorology, and transportation. Specifically, these include: ETTh1\footnote{The ETTh1 dataset was acquired at \url{https://github.com/zhouhaoyi/ETDataset}.}, Exchange\footnote{The Exchange dataset was acquired at \url{https://github.com/laiguokun/multivariate-time-series-data}.}, Weather\footnote{The Weather dataset was acquired at \url{https://www.bgc-jena.mpg.de/wetter/}.}, Appliance\footnote{The Appliance dataset was acquired at \url{https://archive.ics.uci.edu/dataset/374/appliances+energy+prediction}.}, Solar\footnote{The Solar dataset was acquired at \url{https://github.com/laiguokun/multivariate-time-series-data/tree/master/solar-energy}.}, and Traffic\footnote{The Traffic dataset was acquired at \url{https://pems.dot.ca.gov/}.}. The detailed attributes of these datasets are summarized in Table \ref{tab:dataset}. To thoroughly evaluate the model's forecasting capability and stability across different time spans, we set both the historical context length $H$ and the prediction horizon length $L$ within the range of $\left\{96, 168, 336, 720\right\}$ for all datasets.

\begin{table}[!t]
    \centering
    \caption{Detailed information of the datasets. The split sizes denote the total number of time steps in the training, validation, and test sets, respectively.}
    \label{tab:dataset}
    \small
    \begin{tabular}{l c c c l}
        \toprule
        Dataset & Dimension & Frequency & Split size & Field \\
        \midrule
        ETTh1 & 7 & Hourly & (8361, 2091, 6968) & Energy \\
        Exchange & 8 & Daily & (4248, 849, 2276) & Finance \\
        Weather & 21 & 10min & (29509, 7377, 15809) & Weather \\
        Appliance & 28 & 10min & (11051, 2763, 5920) & Energy \\
        Solar-Energy & 137 & 10min & (33638, 8410, 10512) & Energy \\
        Traffic & 862 & Hourly & (9824, 2456, 5263) & Traffic \\
        \bottomrule
    \end{tabular}
\end{table}

\textbf{Baselines.} To evaluate the proposed framework and the effectiveness of the dynamic guidance mechanism, we compared DynG-Diff against six representative diffusion-based probabilistic time series forecasting methods, including TimeGrad \cite{rasul_autoregressive_2021}, CSDI \cite{tashiro_csdi_2021}, SSSD \cite{alcaraz_diffusion-based_2023}, TimeDiff \cite{shen_multi-resolution_2024}, TMDM \cite{li_transformer-modulated_2024}, and D3U \cite{li_diffusion-based_2025}.

\textbf{Evaluation Metrics.} We use the Continuous Ranked Probability Score (CRPS) and Mean Squared Error (MSE). The $S$ trajectories $x^{(s)}$ define the empirical predictive distribution $\hat{F}_S=S^{-1}\sum_{s=1}^{S}\delta_{x^{(s)}}$ and point forecast $\hat{y}=S^{-1}\sum_{s=1}^{S}x^{(s)}$.
\begin{equation}
MSE=\frac{1}{N\times D}\sum_{i=1}^{N}\sum_{j=1}^{D}(y_{i,j}-\hat{y}_{i,j})^{2}
\label{eq:18}
\end{equation}
\begin{equation}
CRPS(\hat{F}_S,y)=\frac{1}{S}\sum_{s=1}^{S}|x^{(s)}-y|-\frac{1}{2S^2}\sum_{s=1}^{S}\sum_{r=1}^{S}|x^{(s)}-x^{(r)}|
\label{eq:19}
\end{equation}
where $y_{i,j}$ and $\hat{y}_{i,j}$ are the observation and point forecast at time step $i$ and variable $j$. CRPS is evaluated from the equally weighted empirical ensemble and averaged across all forecast locations and variables.

\textbf{Implementation Details.} DynG-Diff is optimized via a two-stage separated training strategy. The unconditional diffusion backbone and the state-aware policy network are trained for 300 and 30 epochs, respectively, with a learning rate of 0.001. The forward diffusion employs 100 timesteps, a time embedding dimension of 128, and a linear noise schedule ($\beta_{1}=10^{-4}$ to $\beta_{100}=10^{-1}$). During inference, we generate $S=100$ trajectories per prediction horizon and assign $\kappa_s=s/(S+1)$ to trajectory $s$. The resulting ensemble defines $\hat{F}_S$. All experiments are conducted on a single NVIDIA V100 32GB GPU. The core hyperparameters are summarized in Table \ref{tab:hyperparameters}.

\begin{table}[!t]
    \centering
    \caption{Hyperparameters of DynG-Diff.}
    \label{tab:hyperparameters}
    \small
    \begin{tabular}{l c}
        \toprule
        Hyperparameter & Value \\
        \midrule
        Backbone epochs & 300 \\
        Policy epochs & 30 \\
        Learning rate & 0.001 \\
        Time embedding dim & 128 \\
        Diffusion steps & 100 \\
        $\beta$ scheduler & Linear \\
        $\beta_{1}$ & 0.0001 \\
        $\beta_{100}$ & 0.1 \\
        \bottomrule
    \end{tabular}
\end{table}

Due to varying variable dimensions and GPU memory constraints across datasets, we adaptively adjust specific configurations (Table \ref{tab:train_config}). Specifically, the global guidance scale is fine-tuned according to the dataset's signal-to-noise ratio to optimally balance the unconditional prior and observation guidance. Low-dimensional datasets (e.g., ETTh1, Exchange) utilize a lightweight backbone (hidden dimension 64, 3 S4 blocks), whereas high-dimensional datasets (e.g., Solar, Traffic) require higher capacity.

\begin{table}[!t]
    \centering
    \caption{Specific configurations for different datasets. $L$ denotes the prediction horizon length.}
    \label{tab:train_config}
    \small
    \begin{tabular}{l c c c c c c}
        \toprule
        Dataset & Horizon ($L$) & Guidance scale & Samples & Hidden dim & S4 blocks & Batch size \\
        \midrule
        ETTh1     & All & 3 & 100 & 64  & 3 & 128 \\
        Exchange  & All & 4 & 100 & 64  & 3 & 256 \\
        Weather   & All & 2 & 100 & 128 & 6 & 32 \\
        Appliance & All & 2 & 100 & 128 & 6 & 64 \\
        \cmidrule{1-7}
        \multirow{2}{*}{Solar} & 96 & \multirow{2}{*}{5} & \multirow{2}{*}{50} & \multirow{2}{*}{128} & \multirow{2}{*}{6} & 32 \\
        & $\ge 168$ & & & & & 16 \\
        \cmidrule{1-7}
        \multirow{2}{*}{Traffic} & $\le 336$ & \multirow{2}{*}{1} & 100 & \multirow{2}{*}{512} & \multirow{2}{*}{3} & 8 \\
        & 720 & & 70 & & & 4 \\
        \bottomrule
    \end{tabular}
\end{table}

\subsection{Main results}
\label{subsec:5.2}
In this section, we conduct a comprehensive quantitative comparison between the proposed DynG-Diff framework and representative state-of-the-art (SOTA) multivariate time series probabilistic forecasting methods in terms of point forecasting (MSE) and probabilistic forecasting (CRPS) metrics. Table \ref{tab:main_results} demonstrates the highly competitive performance of DynG-Diff across six real-world benchmark datasets. Particularly on the ETTh1 and Appliance datasets (as visualized in Figures \ref{fig:3} and \ref{fig:4}), which feature complex heterogeneity, DynG-Diff achieves average CRPS values of 0.321 and 0.395, respectively. These results are significantly superior to both the strong diffusion-based baseline TMDM (0.454 and 0.574, respectively) and the D3U model employing the latest decoupled architecture (0.425 and 0.563, respectively).

\begin{figure}[!t]
    \centering
    \includegraphics[width=0.48\textwidth]{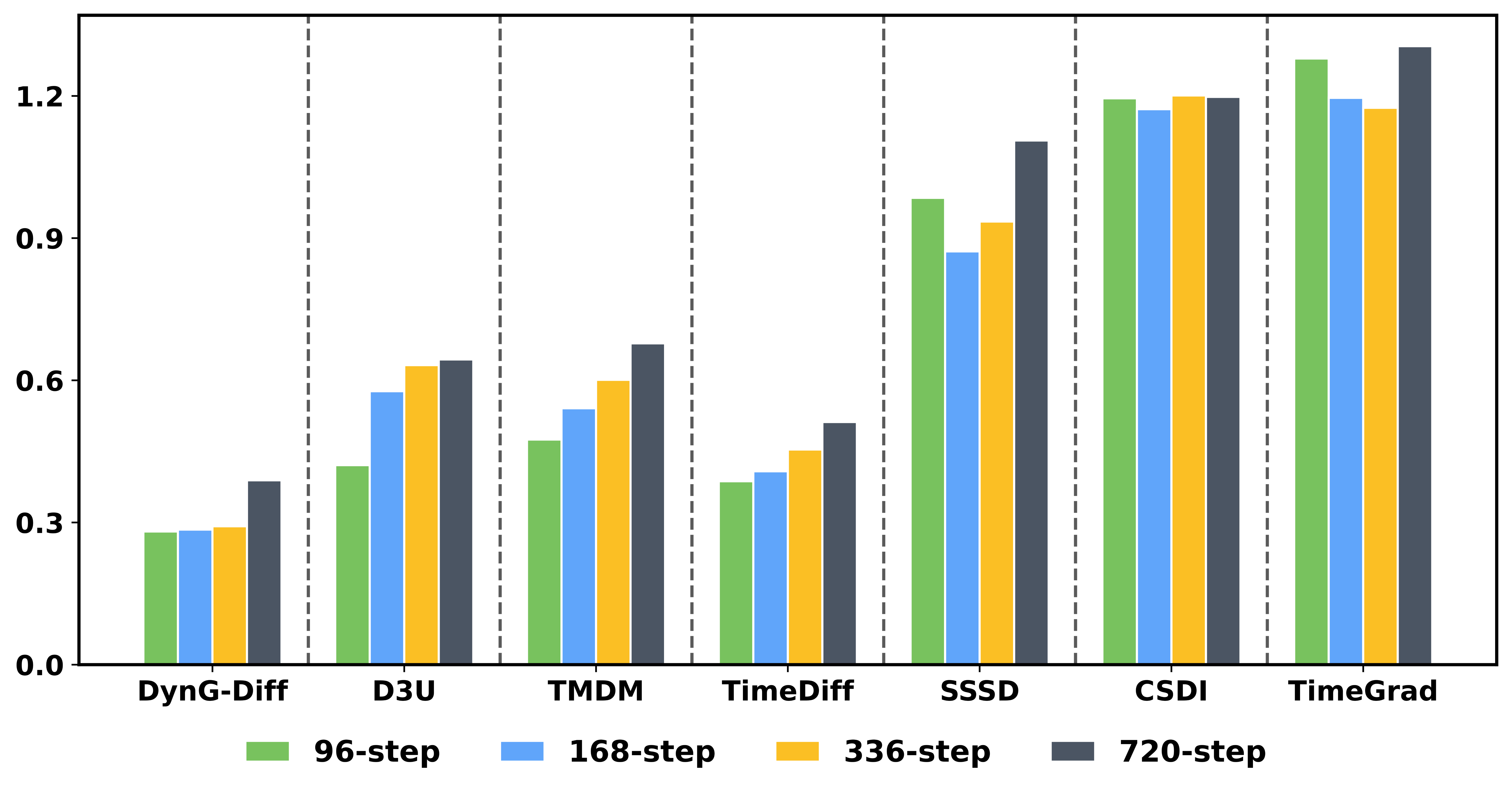}
    \hfill
    \includegraphics[width=0.48\textwidth]{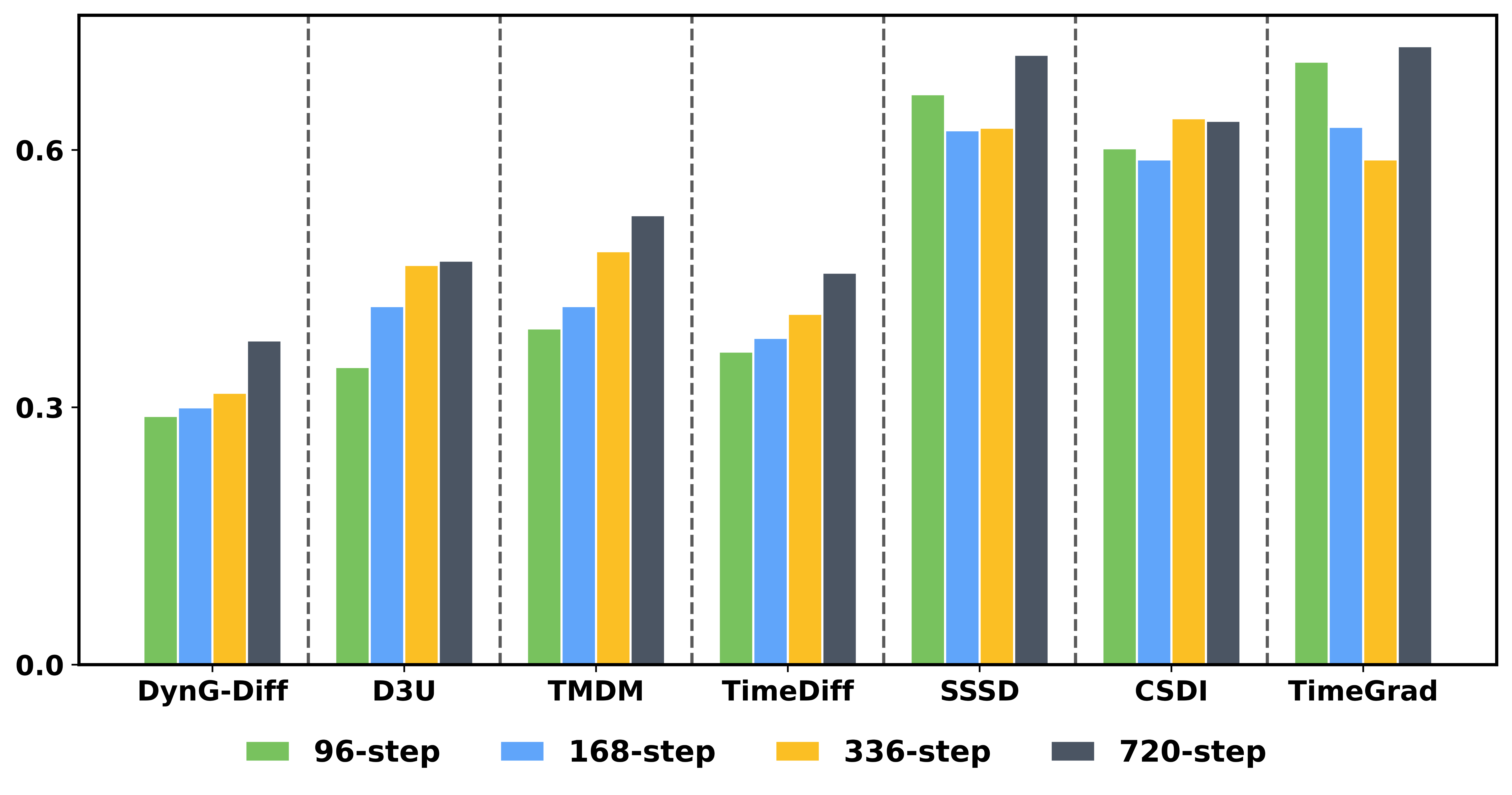}
    \caption{Forecasting results on the ETTh1 dataset: MSE (left) and CRPS (right).}
    \label{fig:3}

    \vspace{0.6em}

    \includegraphics[width=0.48\textwidth]{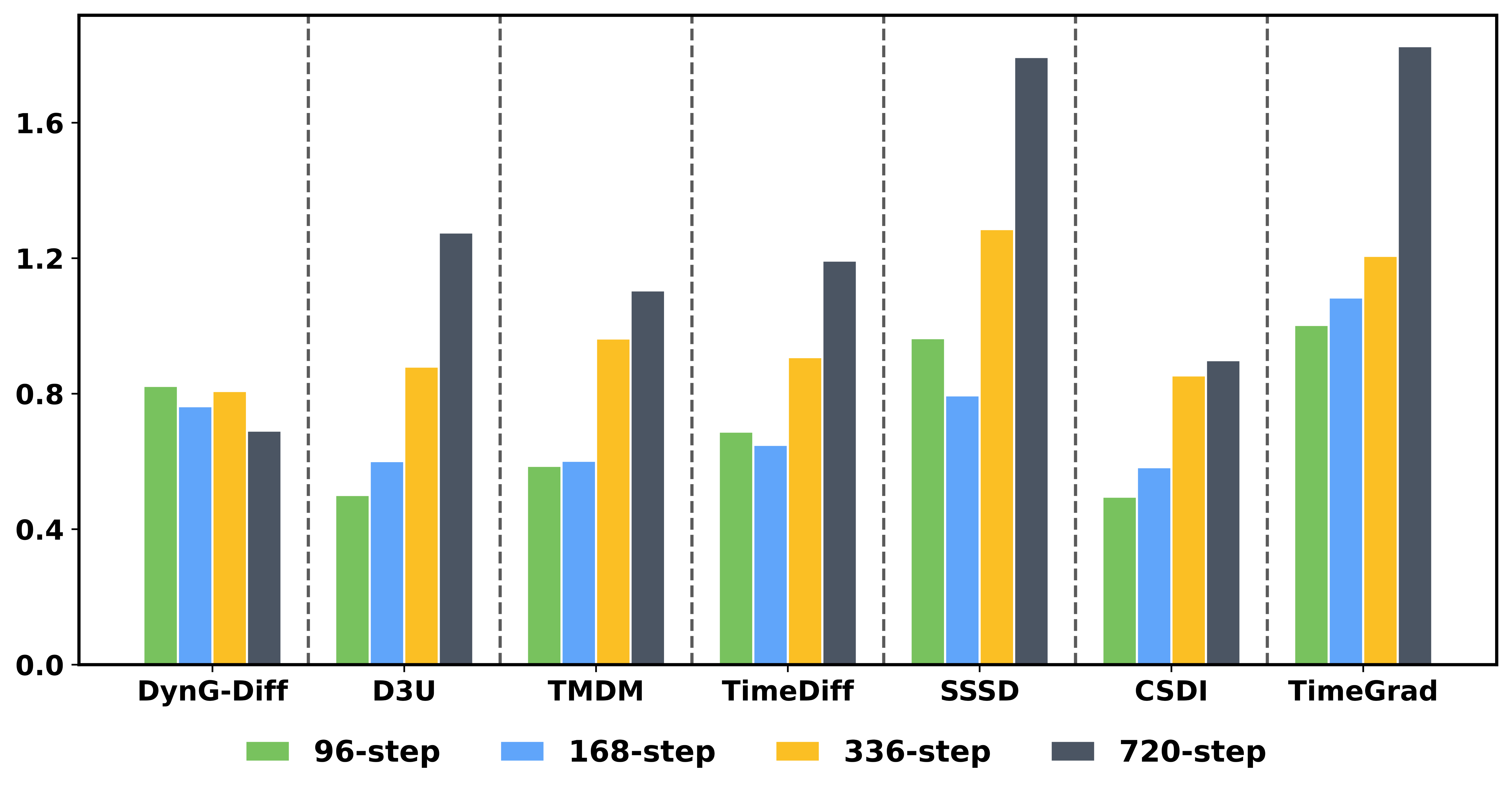}
    \hfill
    \includegraphics[width=0.48\textwidth]{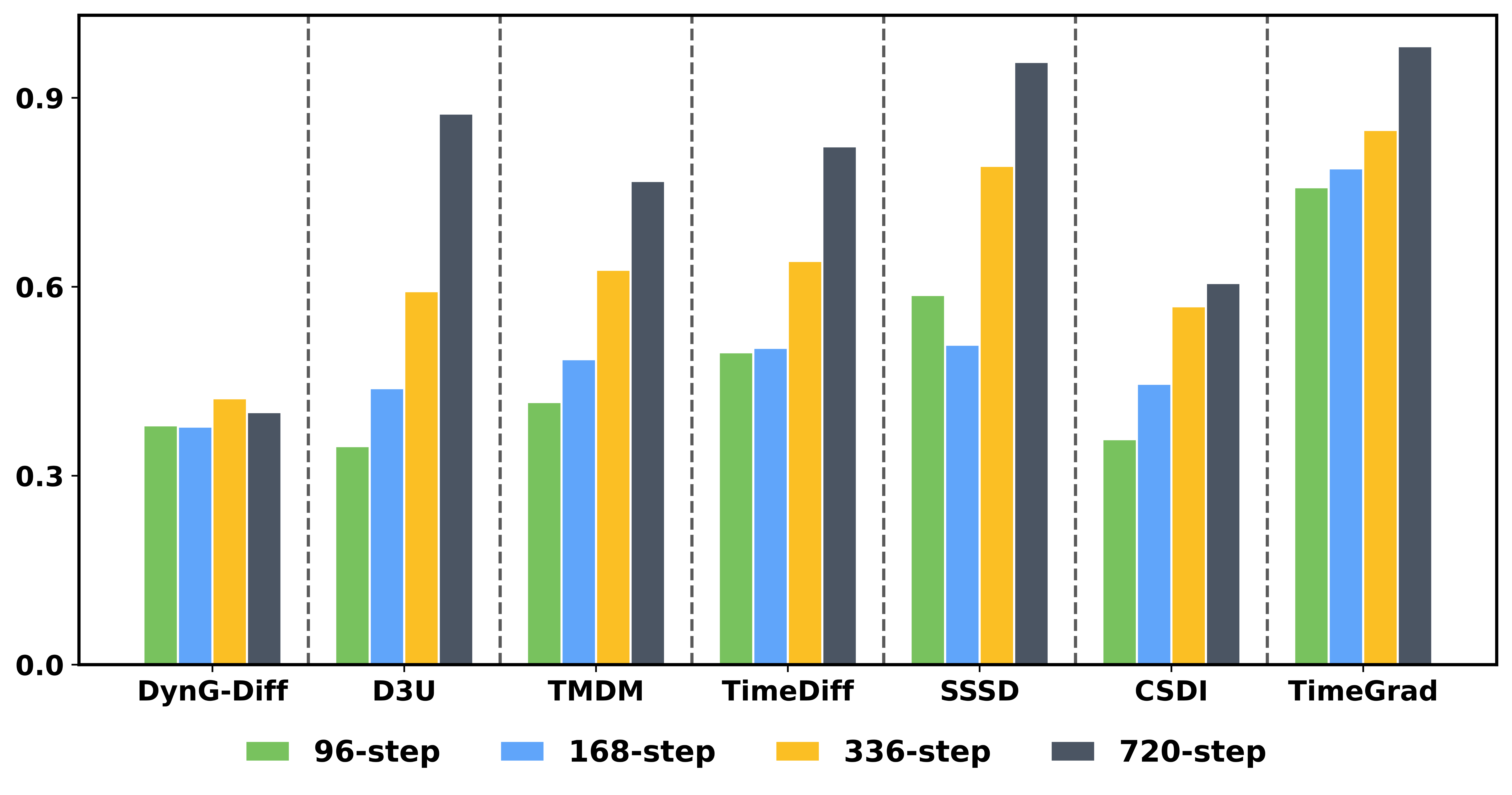}
    \caption{Forecasting results on the Appliance dataset: MSE (left) and CRPS (right).}
    \label{fig:4}
\end{figure}

On Weather, DynG-Diff achieves the lowest average CRPS (0.190), whereas its average MSE (0.502) is higher than those of D3U and TMDM. This divergence indicates a stronger advantage in distributional forecasting than in mean-point estimation. Figure \ref{fig:6} shows localized precision drops across Weather variables and timesteps. These heterogeneous recovery patterns may affect the sample mean more strongly than the predictive distribution, consistent with the MSE--CRPS gap. On Solar, DynG-Diff ranks second across all horizons. The shared diurnal and illumination-driven dynamics produce the vertical block patterns in Figure \ref{sub:d}, favoring methods that explicitly extract regular deterministic trends. This regularity narrows the advantage of variable-specific guidance, while DynG-Diff remains consistently competitive across all horizons.

Conditional diffusion forecasting methods like TMDM tend to entrust the overall data distribution—encompassing trends, seasonality, and extreme noise—to the diffusion model for indiscriminate global fitting. This homogeneous constraint strategy overlooks the significant information heterogeneity across different channels in multivariate sequences. Consequently, the model is highly susceptible to interference from high-noise or anomalous variables during inference generation, which in turn generates adversarial gradients and disrupts the multivariate joint distribution. D3U utilizes a front-end deterministic point forecasting model to extract high-certainty components, relying solely on the diffusion model to fit the high-uncertainty residual distribution. This physical decoupling paradigm enables it to achieve excellent performance in most evaluation scenarios. However, it is noteworthy that when the prediction horizon is extended to 720 steps, D3U's CRPS on the ETTh1 and Appliance datasets deteriorates to 0.471 and 0.875, respectively. In long-term forecasting scenarios, the prediction errors of the upstream deterministic model amplify sharply over time. D3U rigidly attributes all these accumulated cognitive biases to "high-uncertainty residuals" and forces the downstream diffusion model to fit them, ultimately resulting in the unconditional amplification of errors and the breakdown of probabilistic forecasting boundaries.

Overall, the variable-sensitive mechanism benefits most settings by adapting the observation guidance to local precision. However, DynG-Diff does not dominate every Traffic setting. At $L=720$, its MSE is close to that of D3U (0.626 versus 0.610), whereas its CRPS is higher (0.362 versus 0.289). This suggests that the main limitation lies in long-horizon probabilistic calibration rather than point estimation. Traffic contains 862 variables, so the long horizon may make dense reliability estimation more difficult. The absence of explicit sensor-topology modeling may also limit cross-variable calibration in this setting.

\clearpage
\begin{table}[p]
    \centering
    \caption{Comparison of MSE and CRPS on six real-world datasets. The prediction horizon length $L \in \{96, 168, 336, 720\}$. The best and second-best results are highlighted in bold and underlined, respectively.}
    \label{tab:main_results}
    \resizebox{\textwidth}{!}{%
    \begin{tabular}{l l l l l l l l l l l l l l l l}
        \toprule
        \multicolumn{2}{l}{Methods} & \multicolumn{2}{l}{DynG-Diff} & \multicolumn{2}{l}{D3U} & \multicolumn{2}{l}{TMDM} & \multicolumn{2}{l}{TimeDiff} & \multicolumn{2}{l}{SSSD} & \multicolumn{2}{l}{CSDI} & \multicolumn{2}{l}{TimeGrad} \\
        \cmidrule(lr){1-2} \cmidrule(lr){3-4} \cmidrule(lr){5-6} \cmidrule(lr){7-8} \cmidrule(lr){9-10} \cmidrule(lr){11-12} \cmidrule(lr){13-14} \cmidrule(lr){15-16}
        \multicolumn{2}{l}{Metrics} & MSE & CRPS & MSE & CRPS & MSE & CRPS & MSE & CRPS & MSE & CRPS & MSE & CRPS & MSE & CRPS \\
        \midrule
        \multirow{5}{*}{\rotatebox{90}{ETTh1}} 
        & 96  & \textbf{0.281} & \textbf{0.290} & 0.421 & \underline{0.347} & 0.475 & 0.392 & \underline{0.387} & 0.365 & 0.985 & 0.665 & 1.195 & 0.602 & 1.279 & 0.703 \\
        & 168 & \textbf{0.285} & \textbf{0.300} & 0.577 & 0.418 & 0.541 & 0.418 & \underline{0.408} & \underline{0.381} & 0.872 & 0.623 & 1.172 & 0.589 & 1.196 & 0.627 \\
        & 336 & \textbf{0.292} & \textbf{0.317} & 0.632 & 0.466 & 0.601 & 0.482 & \underline{0.454} & \underline{0.409} & 0.935 & 0.626 & 1.201 & 0.637 & 1.175 & 0.589 \\
        & 720 & \textbf{0.389} & \textbf{0.378} & 0.644 & 0.471 & 0.678 & 0.524 & \underline{0.512} & \underline{0.457} & 1.106 & 0.711 & 1.198 & 0.634 & 1.305 & 0.721 \\
        \cmidrule{2-16}
        & Avg & \textbf{0.311} & \textbf{0.321} & 0.568 & 0.425 & 0.573 & 0.454 & \underline{0.440} & \underline{0.403} & 0.974 & 0.656 & 1.191 & 0.615 & 1.238 & 0.660 \\
        \midrule
        \multirow{5}{*}{\rotatebox{90}{Exchange}} 
        & 96  & \textbf{0.120} & \textbf{0.178} & \textbf{0.111} & \underline{0.180} & 0.177 & 0.261 & \underline{0.115} & 0.207 & 0.467 & 0.398 & 0.225 & 0.273 & 1.315 & 0.872 \\
        & 168 & \textbf{0.193} & \textbf{0.231} & \underline{0.198} & \underline{0.257} & 0.254 & 0.315 & 0.211 & 0.292 & 0.465 & 0.397 & 0.484 & 0.447 & 1.304 & 0.876 \\
        & 336 & \textbf{0.352} & \textbf{0.317} & \underline{0.456} & \underline{0.359} & 0.396 & 0.425 & 0.479 & 0.443 & 0.498 & 0.421 & 0.410 & 0.395 & 1.557 & 0.982 \\
        & 720 & \textbf{0.404} & \textbf{0.369} & 0.687 & \underline{0.479} & 0.726 & 0.685 & \underline{0.678} & 0.579 & 0.703 & 0.615 & 0.921 & 0.709 & 1.624 & 1.006 \\
        \cmidrule{2-16}
        & Avg & \textbf{0.267} & \textbf{0.273} & \underline{0.363} & \underline{0.318} & 0.388 & 0.421 & 0.370 & 0.380 & 0.533 & 0.457 & 0.510 & 0.456 & 1.450 & 0.934 \\
        \midrule
        \multirow{5}{*}{\rotatebox{90}{Weather}} 
        & 96  & 0.484 & \textbf{0.178} & \textbf{0.186} & \underline{0.197} & \underline{0.255} & 0.206 & 0.351 & 0.326 & 0.581 & 0.381 & 0.501 & 0.335 & 0.628 & 0.413 \\
        & 168 & 0.522 & \textbf{0.191} & \textbf{0.217} & \underline{0.208} & \underline{0.276} & 0.242 & 0.352 & 0.338 & 0.488 & 0.326 & 0.510 & 0.347 & 0.645 & 0.438 \\
        & 336 & 0.521 & \textbf{0.199} & \textbf{0.259} & \underline{0.227} & \underline{0.324} & 0.280 & 0.359 & 0.323 & 0.461 & 0.313 & 0.493 & 0.318 & 0.497 & 0.368 \\
        & 720 & 0.483 & \textbf{0.195} & 0.438 & 0.359 & \underline{0.347} & 0.298 & 0.371 & 0.384 & 0.625 & 0.433 & \textbf{0.343} & \underline{0.275} & 0.519 & 0.375 \\
        \cmidrule{2-16}
        & Avg & 0.502 & \textbf{0.190} & \textbf{0.275} & \underline{0.247} & \underline{0.300} & 0.256 & 0.358 & 0.342 & 0.538 & 0.363 & 0.461 & 0.318 & 0.572 & 0.398 \\
        \midrule
        \multirow{5}{*}{\rotatebox{90}{Appliance}} 
        & 96  & 0.823 & 0.380 & \underline{0.501} & \textbf{0.347} & 0.587 & 0.417 & 0.688 & 0.496 & 0.964 & 0.587 & \textbf{0.496} & \underline{0.358} & 1.003 & 0.758 \\
        & 168 & 0.763 & \textbf{0.378} & \underline{0.601} & \underline{0.439} & 0.602 & 0.485 & 0.649 & 0.503 & 0.795 & 0.508 & \textbf{0.583} & 0.446 & 1.084 & 0.788 \\
        & 336 & \textbf{0.808} & \textbf{0.423} & 0.880 & 0.593 & 0.963 & 0.627 & 0.908 & 0.641 & 1.286 & 0.792 & \underline{0.854} & \underline{0.569} & 1.207 & 0.849 \\
        & 720 & \textbf{0.691} & \textbf{0.401} & 1.276 & 0.875 & 1.105 & 0.768 & 1.193 & 0.823 & 1.794 & 0.957 & \underline{0.899} & \underline{0.606} & 1.826 & 0.982 \\
        \cmidrule{2-16}
        & Avg & \underline{0.771} & \textbf{0.395} & 0.814 & 0.563 & 0.814 & 0.574 & 0.859 & 0.615 & 1.209 & 0.711 & \textbf{0.708} & \underline{0.494} & 1.280 & 0.844 \\
        \midrule
        \multirow{5}{*}{\rotatebox{90}{Solar}} 
        & 96  & 0.297 & 0.261 & \textbf{0.227} & \textbf{0.235} & \underline{0.268} & 0.296 & 0.798 & 0.587 & 0.397 & 0.368 & 0.358 & 0.329 & 0.589 & 0.408 \\
        & 168 & 0.328 & \underline{0.271} & \textbf{0.273} & \textbf{0.268} & \underline{0.268} & 0.294 & 0.673 & 0.462 & 0.878 & 0.625 & 0.387 & 0.346 & 0.711 & 0.493 \\
        & 336 & 0.457 & \underline{0.324} & \textbf{0.218} & \textbf{0.231} & 0.410 & 0.341 & 0.980 & 0.698 & 0.794 & 0.578 & \underline{0.402} & 0.358 & 0.795 & 0.568 \\
        & 720 & \underline{0.273} & \underline{0.314} & \textbf{0.221} & \textbf{0.234} & 0.462 & 0.358 & 1.195 & 0.802 & 0.801 & 0.579 & 0.399 & 0.354 & 1.002 & 0.759 \\
        \cmidrule{2-16}
        & Avg & \underline{0.338} & \underline{0.292} & \textbf{0.234} & \textbf{0.242} & 0.352 & 0.322 & 0.911 & 0.637 & 0.717 & 0.537 & 0.386 & 0.346 & 0.774 & 0.557 \\
        \midrule
        \multirow{5}{*}{\rotatebox{90}{Traffic}} 
        & 96  & \textbf{0.460} & \underline{0.310} & \underline{0.547} & \textbf{0.241} & 0.603 & 0.445 & 0.784 & 0.537 & 0.809 & 0.471 & 0.879 & 0.518 & 1.008 & 0.561 \\
        & 168 & \textbf{0.280} & \textbf{0.244} & \underline{0.556} & \underline{0.249} & 0.578 & 0.435 & 0.825 & 0.554 & 0.924 & 0.587 & 1.047 & 0.609 & 0.959 & 0.538 \\
        & 336 & \textbf{0.412} & \underline{0.282} & \underline{0.594} & \textbf{0.273} & 0.603 & 0.487 & 0.746 & 0.519 & 0.868 & 0.433 & 1.026 & 0.605 & 1.304 & 0.677 \\
        & 720 & \underline{0.626} & \underline{0.362} & \textbf{0.610} & \textbf{0.289} & 0.833 & 0.572 & 0.983 & 0.677 & 1.072 & 0.609 & 1.240 & 0.674 & 1.103 & 0.582 \\
        \cmidrule{2-16}
        & Avg & \textbf{0.444} & \underline{0.299} & \underline{0.576} & \textbf{0.263} & 0.654 & 0.484 & 0.834 & 0.571 & 0.918 & 0.525 & 1.048 & 0.601 & 1.093 & 0.589 \\
        \bottomrule
    \end{tabular}%
    }
\end{table}
\clearpage

\subsection{Ablation studies}
\label{subsec:5.3}

To investigate the specific contributions of the core components in the DynG-Diff framework, we quantitatively analyzed the model's forecasting performance under different weight modes to verify the superiority of dynamic weight guidance (Dynamic) over homogeneous scalar guidance (Scalar) and no guidance (None). Specifically, by setting the detached importance weight $\bar{a}_{l,d}^t$ in Eq. \eqref{eq:14} to a constant 1, the dynamic weight guidance degenerates into homogeneous scalar guidance; by setting the global guidance scale s to a constant 0, the dynamic weight guidance degenerates into no guidance. Table \ref{tab:ablation_weight} presents the metrics of these three modes across six benchmark datasets under different prediction time scales.

When the guidance mechanism is entirely omitted (NONE), the model's forecasting performance undergoes severe degradation, with its average MSE and CRPS worsening by 125.6\% and 81.7\%, respectively. This significant performance decay indicates that for an unconditionally trained diffusion backbone network, introducing observation guidance during the inference stage is a prerequisite for achieving accurate time series forecasting. Furthermore, when homogeneous guidance is applied, the model's average MSE and CRPS deteriorate by 26.1\% and 14.8\%, respectively. This further demonstrates that different variables in multivariate time series possess distinctly different degrees of importance and information contributions. The "variable-sensitive" dynamic guidance effectively resolves the conflict of variable heterogeneity by intelligently allocating local precision weights to different variables, thereby avoiding the disruption of the overall multivariate joint distribution caused by applying indiscriminate homogeneous guidance.

Figure \ref{fig:5} illustrates the probabilistic forecasting intervals generated by the model on a single channel of the ETTh1 dataset under different guidance modes. It can be observed that the forecasting intervals generated without guidance (NONE) are extremely diffuse, and their median predictions almost completely fail to capture future fluctuation trends. Although homogeneous guidance (SCALAR) can correct the prediction trajectory to some extent, the generated intervals are overly conservative, appearing as a result of "compromise." In contrast, dynamic guidance can apply fine-grained interventions based on the real-time evolutionary state of the system, and its generated probabilistic intervals appear much tighter while ensuring the coverage rate of the ground truth. More probabilistic forecasting results can be found in Appendix \ref{subapp:C.1}.

\begin{figure}[!t]
    \centering
    \begin{subfigure}[b]{0.32\textwidth}
        \centering
        \includegraphics[width=\linewidth]{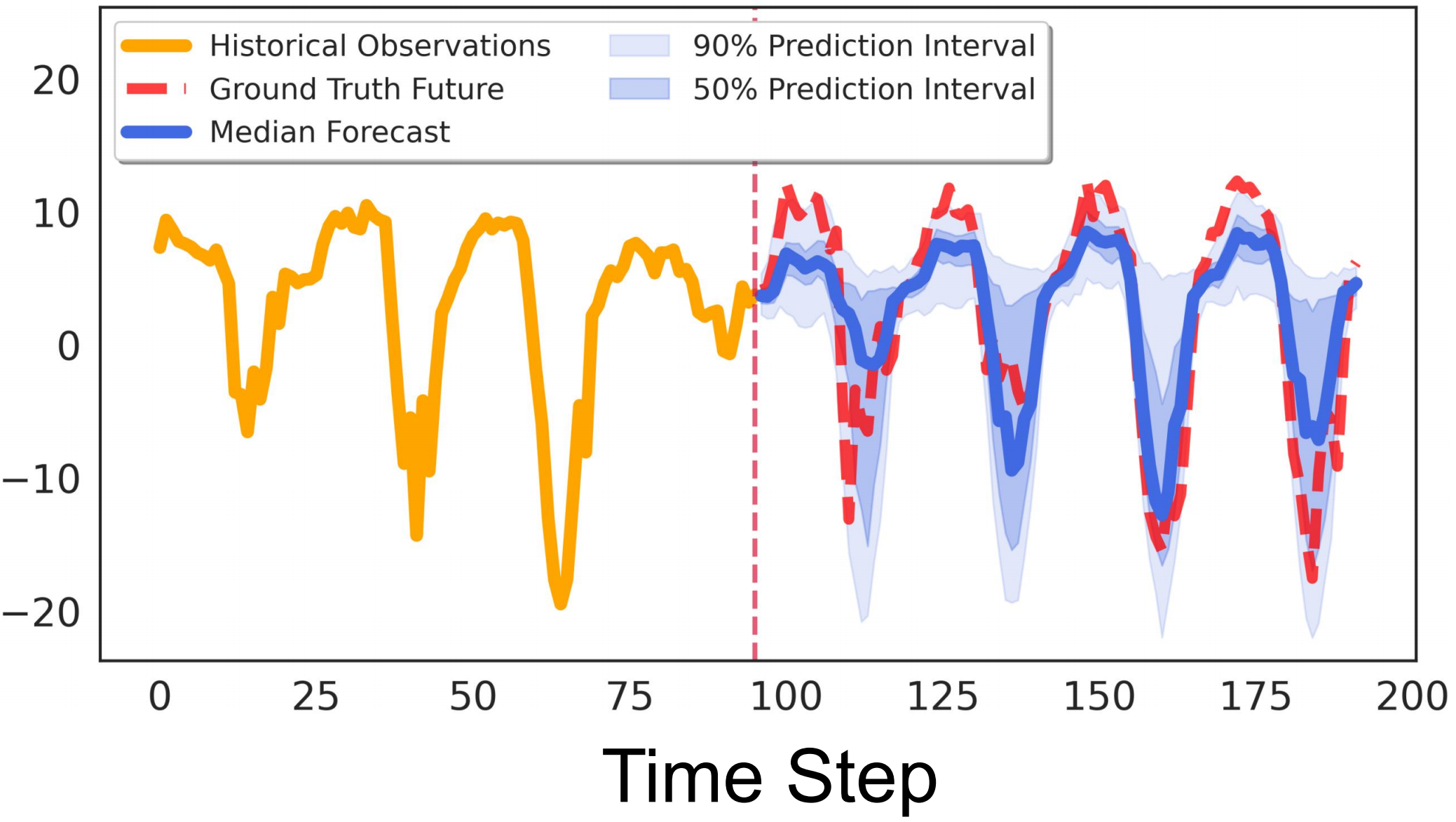}
        \caption{Dynamic Guidance}
    \end{subfigure}
    \hfill
    \begin{subfigure}[b]{0.32\textwidth}
        \centering
        \includegraphics[width=\linewidth]{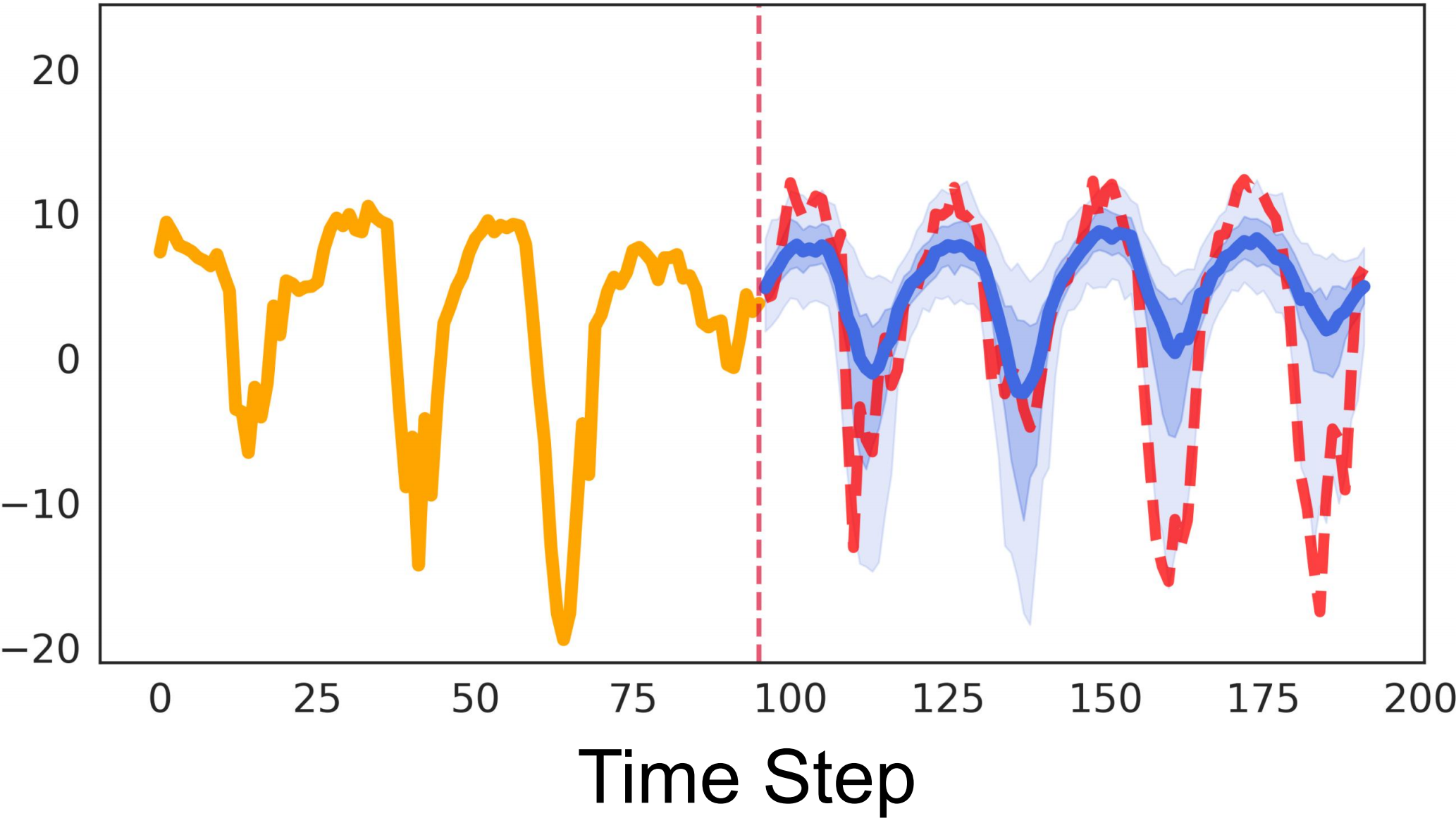}
        \caption{Scalar Guidance}
    \end{subfigure}
    \hfill
    \begin{subfigure}[b]{0.32\textwidth}
        \centering
        \includegraphics[width=\linewidth]{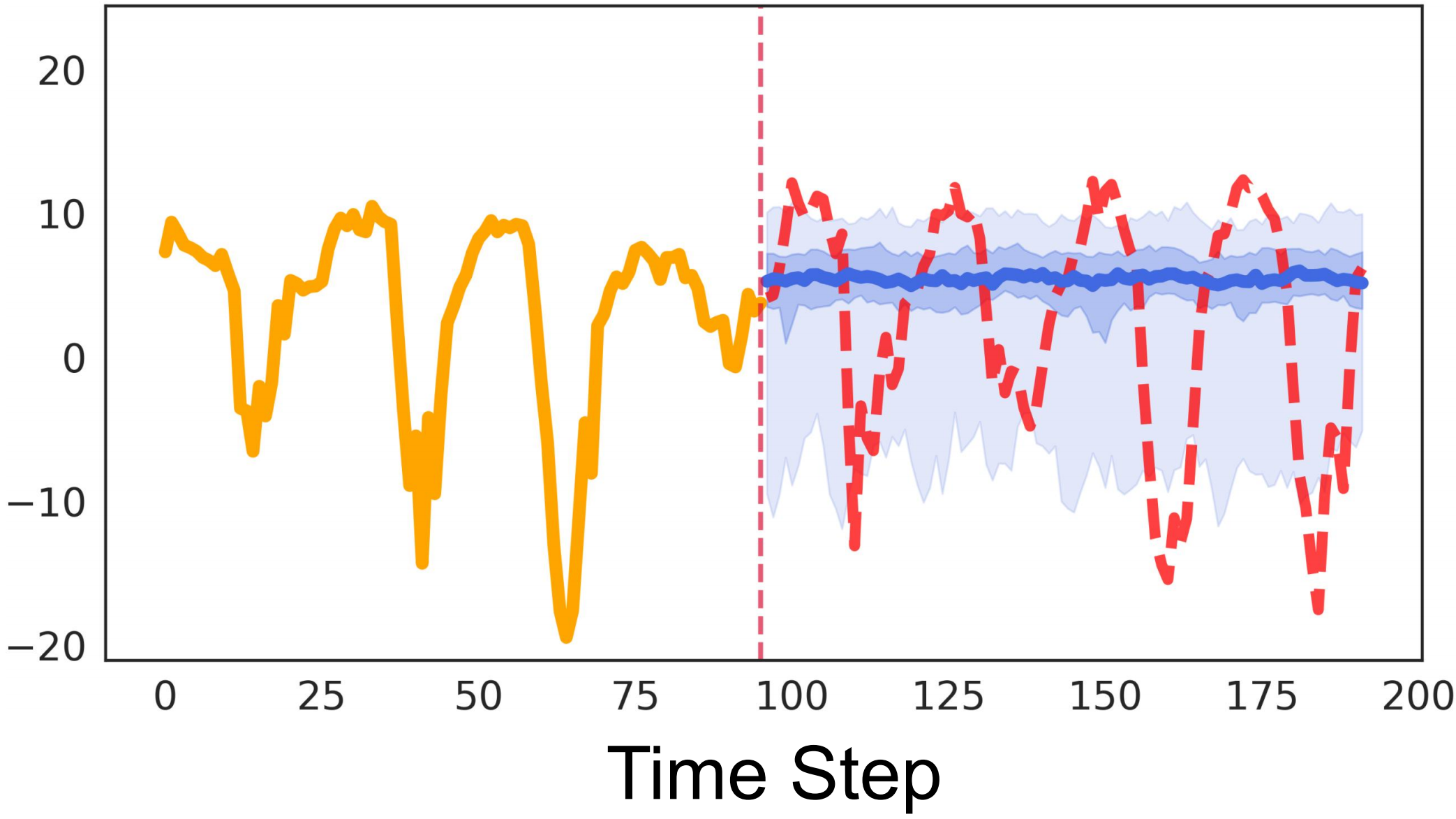}
        \caption{None Guidance}
    \end{subfigure}
    
    \caption{Comparison of probabilistic forecasting intervals generated by the model on the ETTh1 dataset under different weight modes.}
    \label{fig:5}
\end{figure}

\begin{table}[!t]
    \centering
    \caption{Comparison of MSE and CRPS results under different weight modes.}
    \label{tab:ablation_weight}
    \small
    \setlength{\tabcolsep}{4.5pt}
    \begin{tabular}{l l l l l l l l}
        \toprule
        \multicolumn{2}{l}{Weight Mode} & \multicolumn{2}{l}{Dynamic} & \multicolumn{2}{l}{Scalar} & \multicolumn{2}{l}{None} \\
        \cmidrule(lr){1-2} \cmidrule(lr){3-4} \cmidrule(lr){5-6} \cmidrule(lr){7-8}
        \multicolumn{2}{l}{Metrics} & MSE & CRPS & MSE & CRPS & MSE & CRPS \\
        \midrule
        \multirow{4}{*}{\rotatebox{90}{ETTh1}} 
        & 96  & \textbf{0.281} & \textbf{0.290} & \underline{0.315} & \underline{0.307} & 0.647 & 0.468 \\
        & 168 & \textbf{0.285} & \textbf{0.300} & \underline{0.311} & \underline{0.325} & 0.641 & 0.474 \\
        & 336 & \textbf{0.292} & \textbf{0.317} & \underline{0.294} & \underline{0.324} & 0.684 & 0.486 \\
        & 720 & \textbf{0.389} & \textbf{0.378} & \underline{0.407} & \underline{0.400} & 0.772 & 0.532 \\
        \midrule
        \multirow{4}{*}{\rotatebox{90}{Exchange}} 
        & 96  & \underline{0.121} & \underline{0.179} & \textbf{0.120} & \textbf{0.178} & 0.162 & 0.208 \\
        & 168 & \textbf{0.193} & \textbf{0.231} & \underline{0.200} & \underline{0.234} & 0.249 & 0.269 \\
        & 336 & \textbf{0.351} & \underline{0.319} & \underline{0.352} & \textbf{0.317} & 0.497 & 0.414 \\
        & 720 & \underline{0.404} & \textbf{0.369} & \textbf{0.403} & \underline{0.371} & 1.278 & 0.846 \\
        \midrule
        \multirow{4}{*}{\rotatebox{90}{Weather}} 
        & 96  & \textbf{0.484} & \textbf{0.178} & \underline{0.818} & \underline{0.275} & 1.021 & 0.394 \\
        & 168 & \textbf{0.522} & \textbf{0.191} & 1.683 & 0.388 & \underline{0.989} & \underline{0.385} \\
        & 336 & \textbf{0.521} & \textbf{0.199} & \underline{0.704} & \underline{0.246} & 0.988 & 0.394 \\
        & 720 & \textbf{0.483} & \textbf{0.195} & \underline{0.587} & \underline{0.248} & 1.170 & 0.470 \\
        \midrule
        \multirow{4}{*}{\rotatebox{90}{Appliance}} 
        & 96  & \underline{0.823} & \textbf{0.380} & \textbf{0.804} & \underline{0.403} & 0.982 & 0.722 \\
        & 168 & \underline{0.803} & \underline{0.387} & \textbf{0.763} & \textbf{0.378} & 1.001 & 0.487 \\
        & 336 & \textbf{0.808} & \textbf{0.423} & \underline{0.824} & \underline{0.477} & 1.025 & 0.534 \\
        & 720 & \underline{0.742} & \textbf{0.406} & \textbf{0.691} & \underline{0.406} & 1.108 & 0.632 \\
        \midrule
        \multirow{4}{*}{\rotatebox{90}{Solar}} 
        & 96  & \textbf{0.297} & \textbf{0.261} & \underline{0.378} & \underline{0.303} & 1.497 & 0.690 \\
        & 168 & \textbf{0.328} & \textbf{0.271} & \underline{0.660} & \underline{0.376} & 1.496 & 0.691 \\
        & 336 & \underline{0.457} & \textbf{0.324} & \textbf{0.456} & \underline{0.355} & 1.497 & 0.693 \\
        & 720 & \textbf{0.273} & \textbf{0.314} & \underline{0.389} & \underline{0.393} & 1.471 & 0.689 \\
        \midrule
        \multirow{4}{*}{\rotatebox{90}{Traffic}} 
        & 96  & \textbf{0.460} & \textbf{0.310} & \underline{0.596} & \underline{0.389} & 1.281 & 0.639 \\
        & 168 & \textbf{0.295} & \textbf{0.247} & \underline{0.310} & \underline{0.274} & 1.148 & 0.590 \\
        & 336 & \textbf{0.412} & \textbf{0.282} & \underline{0.444} & \underline{0.298} & 1.138 & 0.581 \\
        & 720 & \underline{0.651} & \textbf{0.371} & \textbf{0.646} & \underline{0.382} & 1.323 & 0.647 \\
        \midrule
        \multicolumn{2}{c}{\textbf{Overall Avg}} & \textbf{0.444} & \textbf{0.296} & \underline{0.560} & \underline{0.340} & 1.002 & 0.538 \\
        \bottomrule
    \end{tabular}
\end{table}

\subsection{Model analysis}
\label{subsec:5.4}

\subsubsection{Spatio-Temporal Consistency and Heterogeneity Perception of Dynamic Weights}
\label{subsubsec:5.4.1}

As illustrated in Figure \ref{fig:6}, we extracted the true single-step observation precision of the Weather dataset at key timesteps during the diffusion denoising process and compared it with the real-time dynamic guidance weights generated by the network through heatmap visualization. It can be observed that the weight matrix generated by the network (bottom half) and the true prediction precision (top half) exhibit a high degree of semantic consistency in their spatio-temporal distributions. High-precision regions (bright yellow) in the heatmaps often precisely correspond to strong guidance weights, whereas low-precision or high-uncertainty regions (dark black) are assigned extremely weak guidance intensities. The dynamic weight network is also able to acutely perceive and adaptively respond to local anomalous features, such as the sudden drop in precision in certain channels present in the Weather dataset. This result demonstrates that the policy network can not only distinguish the importance differences among various variables in the spatial dimension but also discern the reliability evolution of the denoising state as the timesteps progress. This aligns well with our initial design objective: applying strong gradient guidance to high-confidence variables to stabilize the recovery path of the overall system, while simultaneously applying weak guidance to low-confidence (unreliable) variables to effectively suppress their negative interference with the denoising process. More visualization results can be found in Appendix \ref{subapp:C.2}.

The localized precision drops in Figure \ref{fig:6} show that reliability varies across both channels and time, providing a qualitative basis for analyzing the response of dynamic guidance to channel corruption.

\begin{figure}[!t]
	\centering
	\includegraphics[width=.88\linewidth]{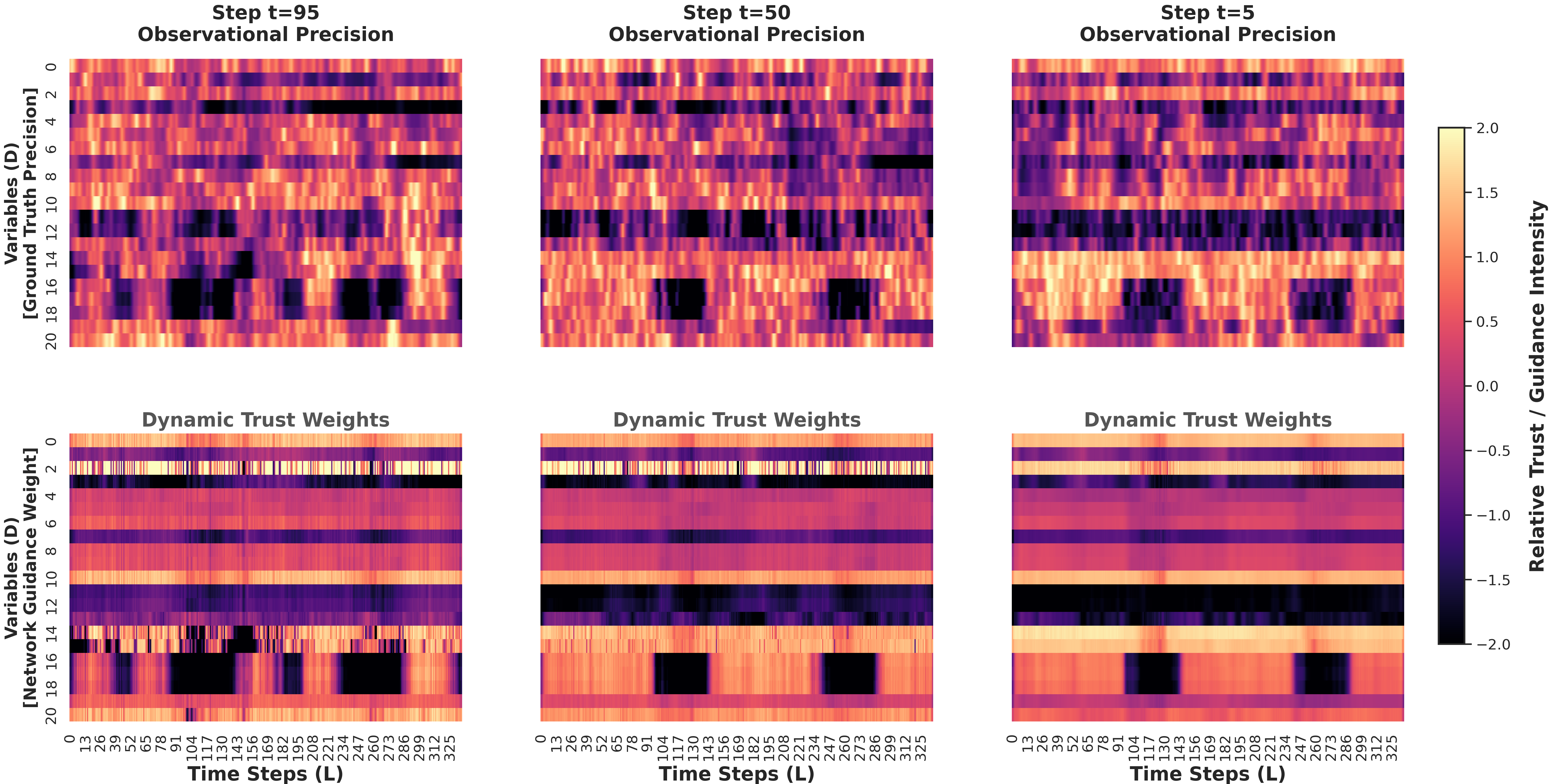}
	\caption{Heatmap comparison between the true observation precision and the dynamic guidance weights generated by the network at key timesteps of the diffusion process in the Weather dataset.}
	\label{fig:6}
\end{figure}

\subsubsection{Robustness Evaluation under Extreme Noise Scenarios}
\label{subsubsec:5.4.2}

Based on the conclusion that "the dynamic network can effectively suppress interference from unreliable variables," we further conducted stress tests to evaluate the model's stability when facing sensor damage or extreme high-frequency noise. We injected independent standard Gaussian noise into partially specified channels of the ETTh1 and Weather datasets, and comprehensively compared the performance degradation rates of dynamic weight guidance and homogeneous scalar guidance on the MSE and CRPS metrics. The results are shown in Table \ref{tab:noise_robustness}, Figure \ref{fig:7}.

Under injected noise, homogeneous guidance degrades sharply because it constrains clean and corrupted variables equally. Dynamic guidance instead reduces the influence of unreliable channels and limits the propagation of misleading gradients. On Weather, the negative MSE degradation should not be interpreted as a benefit of noise. Reweighting corrupted channels can incidentally shift the finite-sample mean closer to the ground truth, whereas the higher CRPS shows that distributional quality still deteriorates. This behavior is consistent with the reliability-suppression pattern in Figure \ref{fig:6}, although the current evidence does not establish a channel-level causal relationship.

\begin{table}[!t]
    \centering
    \caption{Comparison of MSE and CRPS degradation between dynamic weight guidance and homogeneous weight guidance models after injecting Gaussian noise into the ETTh1 and Weather datasets.}
    \label{tab:noise_robustness}
    \resizebox{\textwidth}{!}{%
    \begin{tabular}{l l c c c c c c c c}
        \toprule
        \multicolumn{2}{l}{Destruction Mode} & \multicolumn{4}{c}{Dynamic} & \multicolumn{4}{c}{Scalar} \\
        \cmidrule(lr){1-2}\cmidrule(lr){3-6} \cmidrule(lr){7-10}
        \multicolumn{2}{l}{Metrics} & MSE & Degradation & CRPS & Degradation & MSE & Degradation & CRPS & Degradation \\
        \midrule
        \multirow{2}{*}{\begin{tabular}[c]{@{}l@{}}ETTh1\\(Noise Ch5)\end{tabular}} 
        & 96  & 0.302 & 7.47\% & 0.315 & 8.62\% & 0.345 & 7.81\% & 0.340 & 9.67\% \\
        & 168 & 0.319 & 11.92\% & 0.341 & 13.66\% & 0.384 & 24.67\% & 0.389 & 20.06\% \\
        \midrule
        \multirow{2}{*}{\begin{tabular}[c]{@{}l@{}}Weather\\(Noise Ch10, 20)\end{tabular}} 
        & 96  & 0.401 & -17.14\% & 0.207 & 16.29\% & 0.888 & 8.55\% & 0.277 & 23.11\% \\
        & 168 & 0.439 & -15.90\% & 0.229 & 19.89\% & 1.927 & 14.49\% & 0.450 & 25.69\% \\
        \bottomrule
    \end{tabular}
    }
\end{table}

\begin{table}[!t]
    \centering
    \caption{Comparison of parameters and computational efficiency between the unconditional backbone network and the state-aware policy network on three datasets.}
    \label{tab:efficiency}
    \resizebox{\textwidth}{!}{%
    \begin{tabular}{l ccc ccc ccc}
        \toprule
        Dataset & \multicolumn{3}{c}{ETTh1} & \multicolumn{3}{c}{Weather} & \multicolumn{3}{c}{Traffic} \\
        \cmidrule(lr){1-1} \cmidrule(lr){2-4} \cmidrule(lr){5-7} \cmidrule(lr){8-10}
        Model Component & 
        \begin{tabular}[c]{@{}c@{}}Params\\(MB)\end{tabular} & \begin{tabular}[c]{@{}c@{}}Train\\(min)\end{tabular} & \begin{tabular}[c]{@{}c@{}}Infer\\(min)\end{tabular} & 
        \begin{tabular}[c]{@{}c@{}}Params\\(MB)\end{tabular} & \begin{tabular}[c]{@{}c@{}}Train\\(min)\end{tabular} & \begin{tabular}[c]{@{}c@{}}Infer\\(min)\end{tabular} & 
        \begin{tabular}[c]{@{}c@{}}Params\\(MB)\end{tabular} & \begin{tabular}[c]{@{}c@{}}Train\\(min)\end{tabular} & \begin{tabular}[c]{@{}c@{}}Infer\\(min)\end{tabular} \\
        \midrule
        Backbone       & 0.19 & 13.62 & 3.06 & 0.95 & 21.62 & 30.72 & 8.60 & 14.65 & 28.62 \\
        Policy Network & 0.02 & 1.08  & /    & 0.03 & 1.22  & /     & 0.41 & 1.50  & /     \\
        DynG-Diff      & 0.21 & 14.70 & 4.18 & 0.98 & 22.84 & 39.74 & 9.01 & 16.15 & 37.31 \\
        \bottomrule
    \end{tabular}
    }
\end{table}

\begin{figure}[!t]
	\centering
	\includegraphics[width=0.78\linewidth]{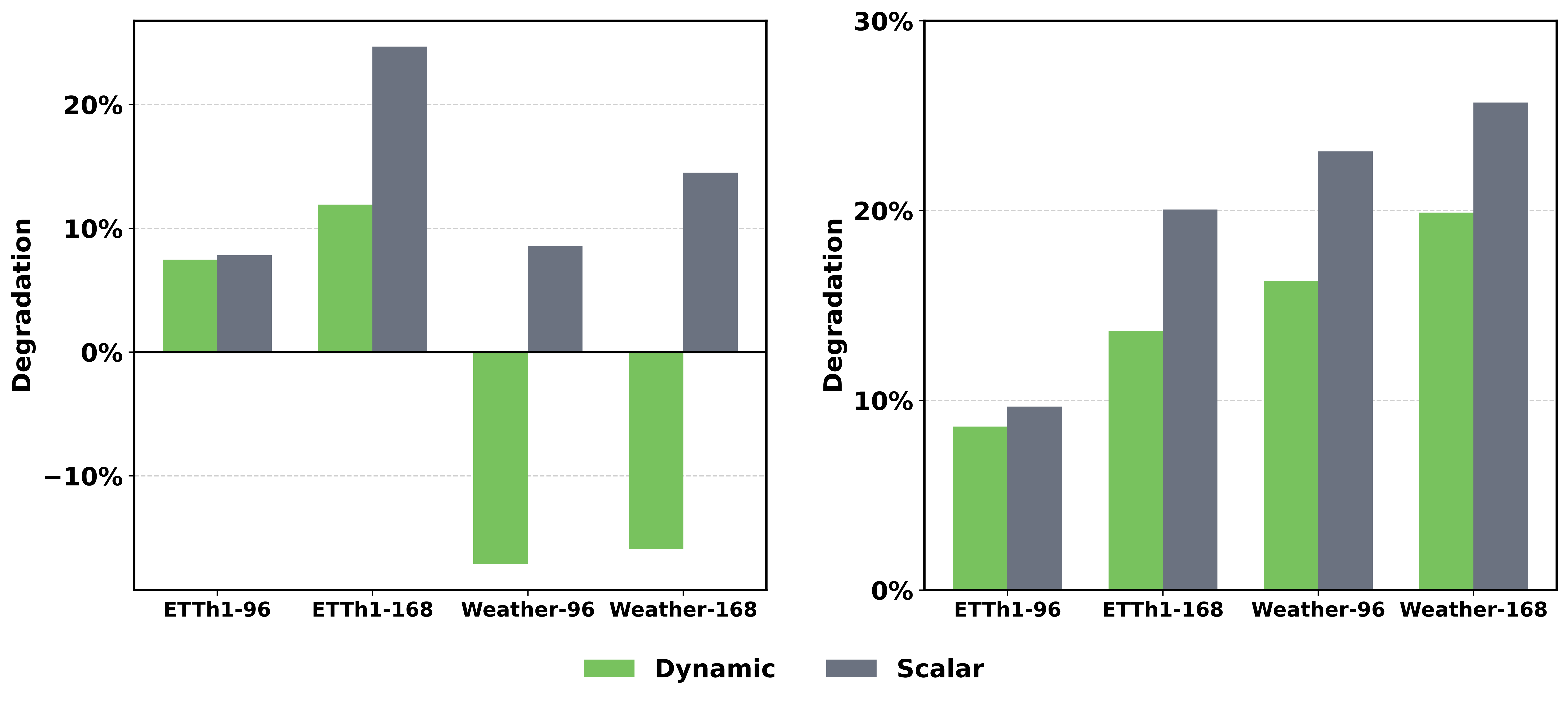}
	\caption{Comparison of degradation rates under extreme noise for different guidance mechanisms: MSE (left) and CRPS (right).}
	\label{fig:7}
\end{figure}

\subsubsection{Computation efficiency analysis}
\label{subsubsec:5.4.3}

We quantitatively analyzed the computational overhead of DynG-Diff on ETTh1, Weather, and Traffic. Table \ref{tab:efficiency} and Figure \ref{fig:8} compare parameter counts, training time, and inference time over the full datasets at $L=96$. The lightweight policy network accounts for approximately 3\% to 10\% of the backbone parameters.

The policy network adds 1.5 minutes of training on Traffic. During inference, however, dynamic guidance increases the total time from 28.62 to 37.31 minutes, an overhead of approximately 30.4\%. This cost is non-negligible for latency-sensitive applications.

The overhead mainly arises from computing the weighted observation-likelihood gradient at each reverse step. The current implementation is therefore better suited to offline or batch forecasting. Future work will explore selective guidance steps and faster diffusion solvers.

\begin{figure}[!t]
	\centering
	\includegraphics[width=0.78\linewidth]{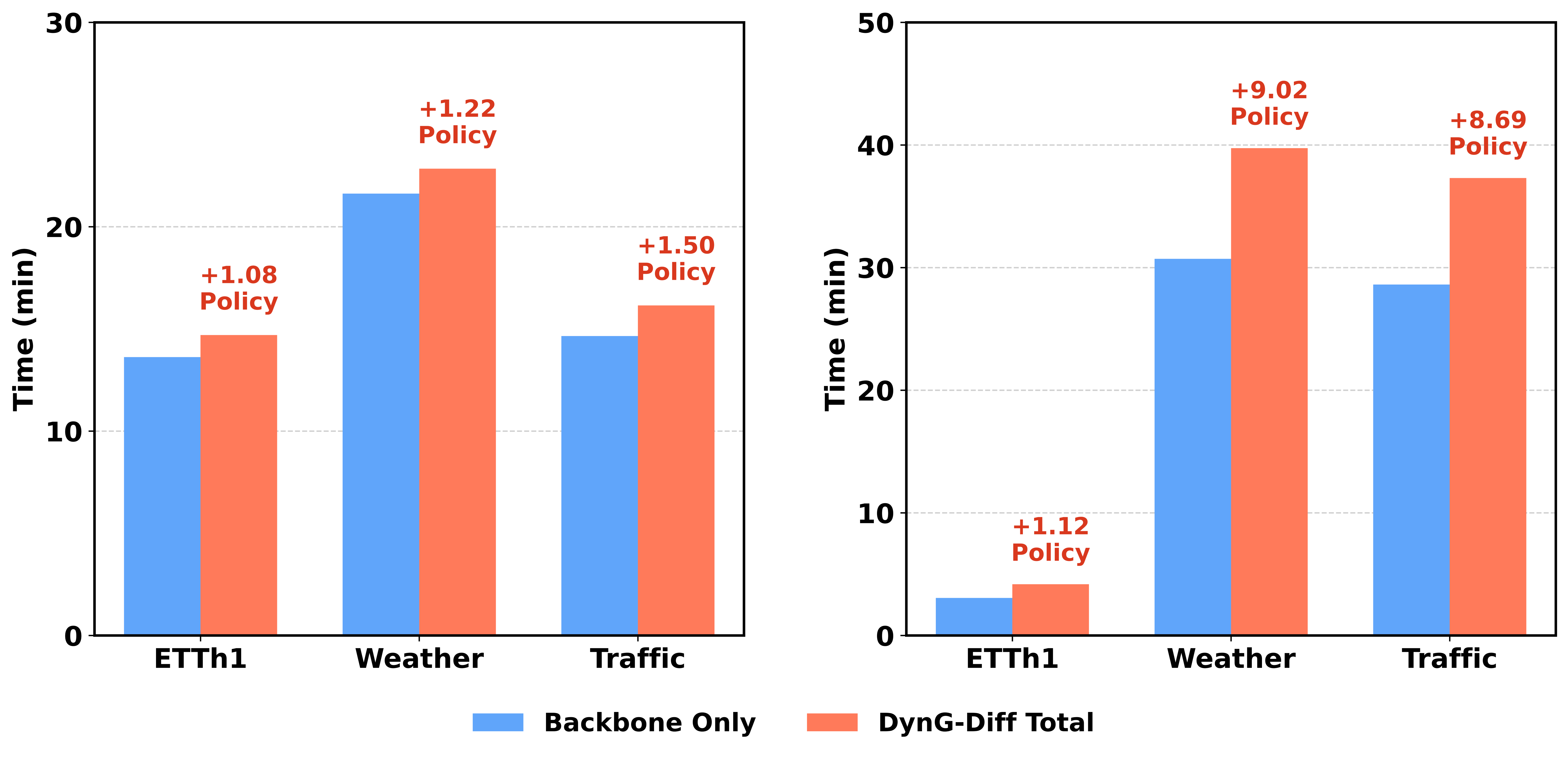}
	\caption{Comparison of computational overhead (in minutes) across different datasets: Train (left) and Inference (right).}
	\label{fig:8}
\end{figure}

\section{Conclusion}
\label{sec:Conclusion}

In this paper, we proposed DynG-Diff, a variable-sensitive dynamic guidance diffusion framework designed to address the challenge of information heterogeneity in probabilistic multivariate time series forecasting. The framework separates unconditional joint-distribution learning from variable-sensitive observation guidance during inference. Its lightweight policy network adaptively infers the reliability of each variable from real-time noisy states and one-step denoising estimates. By formulating the dynamic guidance matrix as the local precision of an Asymmetric Laplace Distribution, DynG-Diff applies stronger guidance to high-confidence variables while suppressing interference from anomalous noise. Experiments on multiple real-world datasets demonstrate competitive performance in most scenarios, while stress tests validate its robustness under severe noise.

Building on the probabilistic forecasting capability demonstrated in this study, future work will investigate extensions to missing-value imputation and anomaly detection through task-appropriate guidance mechanisms. Additional directions include accelerating reverse diffusion sampling and integrating topological structures to model spatial dependencies among heterogeneous variables. Overall, this work provides a basis for studying reusable unconditional generative backbones in multivariate time-series analysis.

\appendix
\numberwithin{equation}{section}
\numberwithin{figure}{section}
\numberwithin{table}{section}

\section{Supplementary Derivations of Denoising Diffusion Probabilistic Models}
\label{app:A}

In Section \ref{subsec:3.2}, we briefly introduced the core objective function of Denoising Diffusion Probabilistic Models (DDPMs). In this section, we provide supplementary explanations of the core theoretical formulations in both the forward and reverse processes, which serve as the foundation for constructing the dynamic guidance mechanism.

\textbf{Forward Process and Arbitrary-Step Sampling.} During the forward noise-addition process, the model progressively injects Gaussian noise into the initial data $x^{0}$ through a fixed Markov chain. By leveraging the reparameterization trick and the additivity of independent Gaussian distributions, the latent state $x^{t}$ at an arbitrary time step $t$ can be directly sampled from the initial state $x^{0}$ via a single-step computation:
\begin{equation}
x^{t} = \sqrt{\bar{\alpha}_{t}}x^{0} + \sqrt{1 - \bar{\alpha}_{t}}\epsilon
\label{eq:20}
\end{equation}
where $\bar{\alpha}_{t} = \prod_{i=1}^{t} \alpha_{i}$ and the merged noise is $\epsilon \sim \mathcal{N}(0, I)$. From this formulation, we can inversely derive the model's real-time, single-step estimation of the original data $\hat{x}^{0}$, which constitutes the core premise for DynG-Diff to construct the state-aware policy network and calculate the observation guidance loss.

\textbf{Reverse Process and Optimization Objective.} The reverse denoising process aims to progressively reconstruct the true data distribution starting from standard Gaussian noise $x^T \sim \mathcal{N}(0, I)$ by learning the transition distributions. Ho et al. \cite{ho_denoising_2020} demonstrated that by parameterizing and simplifying the Evidence Lower Bound (ELBO), minimizing the KL divergence between the reverse transition distribution and the true posterior distribution is ultimately equivalent to optimizing the simplified mean squared error loss $\mathcal{L}_{simple}$ as shown in Eq. \eqref{eq:2}. During the inference stage, based on the noise $\epsilon_{\theta}(x^t, t)$ predicted by the unconditional diffusion backbone, the single-step reverse denoising sampling process can be formulated as:
\begin{equation}
x^{t-1} = \frac{1}{\sqrt{\alpha_t}} \left( x^t - \frac{1 - \alpha_t}{\sqrt{1 - \bar{\alpha}_t}} \epsilon_{\theta}(x^t, t) \right) + \sigma_t z \quad
\label{eq:21}
\end{equation}
where $z \sim \mathcal{N}(0, I)$ denotes the introduced random noise term. During the guided inference phase of DynG-Diff, we utilize the dynamic weights outputted by the policy network to compute the observation likelihood-based guidance gradient $\nabla_{x^t}\mathcal{L}_{guide}$, which is then strictly superimposed onto the aforementioned sampled noise prediction term $\epsilon_{\theta}(x^t, t)$ according to Bayes' theorem, thereby achieving variable-sensitive conditional generation without the need for retraining.

\section{Technical Details}
\label{app:B}

\subsection{Unconditional Backbone Network Architecture and Training Details}
\label{subapp:B.1}

In the DynG-Diff framework, the core role of the unconditional backbone network $\epsilon_\theta$ is to learn the underlying joint distribution of multivariate time series, providing a solid generative prior for dynamic guidance during the inference stage. To balance computational efficiency for long sequence modeling and the ability to capture complex dynamics, the network departs from the traditional Transformer architecture and adopts a residual network design based on Structured State Space models (S4) \cite{gu_efficiently_2021}.

The input to the network at any arbitrary diffusion timestep consists of two parts: one is the current noisy intermediate state $x^t \in \mathbb{R}^{L \times D}$, where $L$ is the sequence length and $D$ is the variable dimension; the other is the current discrete diffusion timestep $t$. The output of the network is the prediction of the Gaussian noise added at that timestep, $\hat{\epsilon} = \epsilon_\theta(x^t, t)$, whose dimension is consistent with the input, i.e., $\mathbb{R}^{L \times D}$.

Because DynG-Diff adopts a decoupled training paradigm, the backbone and policy networks are optimized independently. During pre-training, the backbone remains unconditional and does not use historical observations $y_{obs}$ or forecasting labels. Its objective is the simplified mean squared error loss in Eq. \eqref{eq:2}, through which it approximates the joint distribution represented by the training data. The trained backbone then provides one-step estimates at arbitrary diffusion timesteps, which are combined with observation likelihoods during guided inference.

The internal architecture of the unconditional backbone network primarily consists of three core modules: feature and time mapping, deep temporal feature extraction, and output aggregation, as illustrated in Figure \ref{FIG:9}. After the input $x^{t}$ and $t$ undergo feature and time mapping, $N$ S4 residual blocks extract long-term dependencies. The skip features from each block are globally aggregated and fused with residuals to output the predicted noise. Among them, the S4 layer, acting as the core operator, can efficiently capture long-range physical dependencies in multivariate sequences with near-linear complexity, and is combined with a gating mechanism to filter out redundant noise.

\begin{figure}[!t]
	\centering
	\includegraphics[width=.88\linewidth]{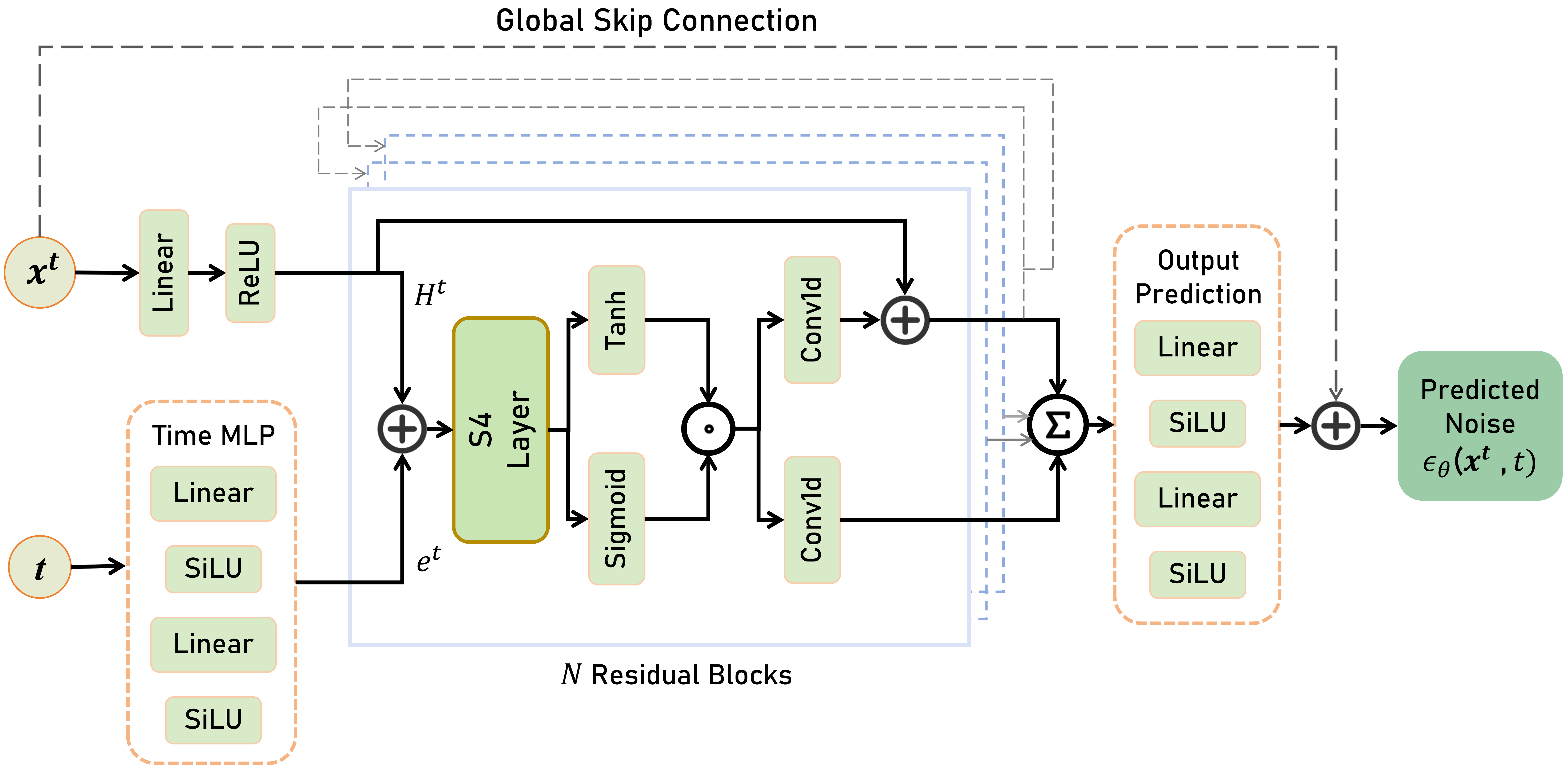}
	\caption{Schematic of the unconditional backbone network architecture and data flow. The network receives the noisy state $x^t$ and timestep $t$ as inputs. The timestep $t$ is encoded into global temporal features $e^t$ via a Time MLP and injected layer by layer into $N$ S4 residual blocks; the noisy state $x^t$ enters the S4 core layer after spatial mapping to extract long-term sequential dependencies. The skip features generated by each residual block are aggregated at a global summation node (Global Skip Sum). After being mapped by the Output Prediction Head, they undergo global residual fusion (Global Residual Add) with the initial state from the input end, ultimately outputting the predicted noise $\hat{\epsilon} = \epsilon_\theta(x^t, t)$.}
	\label{FIG:9}
\end{figure}

\subsection{Asymmetric Laplace Distribution Modeling for Conditional Observation Distribution}
\label{subapp:B.2}

In Section \ref{subsec:4.2}, to accurately delineate the predictive uncertainty in probabilistic forecasting tasks, we adopted the Asymmetric Laplace Distribution (ALD) to model the conditional observation distribution $p(y_{obs}|x^t)$ given the latent variable $x^t$. This section provides the detailed theoretical derivations for this modeling.

For a random variable $Y$ following an Asymmetric Laplace Distribution with a location parameter $\mu$, a scale parameter $b > 0$, and an asymmetry parameter $\kappa \in (0,1)$, its standard probability density function is defined as:
\begin{equation}
f(y; \mu, b, \kappa) = \frac{\kappa(1 - \kappa)}{b}\exp\left(-\frac{\rho_\kappa(y - \mu)}{b}\right) \quad
\label{eq:22}
\end{equation}
where $\rho_\kappa(e) = \max(\kappa \cdot e, (\kappa - 1) \cdot e)$ is the asymmetric quantile loss. For $(l,d)$, $y$ corresponds to $y_{l,d}$, $\mu$ to $\hat{x}_{l,d}^0$, and $1/b$ to the policy output $a_{l,d}^t$. During guidance, we use its detached value $\bar{a}_{l,d}^t=\operatorname{sg}(a_{l,d}^t)$ as the plug-in local precision. Substitution gives:
\begin{equation}
p(y_{l,d}|x^t; \kappa, \bar{a}_{l,d}^t) = \kappa(1 - \kappa)\bar{a}_{l,d}^t \exp\left(- \bar{a}_{l,d}^t \cdot \rho_\kappa(y_{l,d} - \hat{x}_{l,d}^0)\right) \quad
\label{eq:23}
\end{equation}
Because $\bar{a}_{l,d}^t$ is a stop-gradient quantity, $\kappa(1-\kappa)\bar{a}_{l,d}^t$ is constant with respect to $x^t$ during the current local guidance update. It can therefore be omitted when taking this gradient, yielding Eq. \eqref{eq:11}.

\section{Experiments}
\label{app:C}

\subsection{Additional Comparisons of Generated Probabilistic Forecasting Intervals}
\label{subapp:C.1}

In Section \ref{subsec:5.3}, we illustrated the probabilistic forecasting intervals generated by the model on the ETTh1 dataset under different guidance modes. In Figure \ref{fig:10}, we additionally present comparisons of the probabilistic forecasting intervals generated by the model on five datasets: Exchange, Weather, Appliance, Solar, and Traffic. We make targeted adjustments to the historical observation window and the prediction window based on the sampling frequencies of the datasets. For datasets with a "10min" sampling interval (Weather, Appliance, Solar), we set $H=L=168$; for datasets with an hourly or daily sampling interval (Traffic and Exchange, respectively), we set $H=L=96$. Furthermore, across all datasets, we select variable channels exhibiting distinct fluctuation amplitudes and evolutionary patterns for visualization.

From the comparisons across the subfigures, it is clearly observable that the dynamic guidance mechanism consistently generates more compact probabilistic intervals that effectively cover the true observations. In contrast, the prediction intervals under the no-guidance mode exhibit significant dispersion, while homogeneous scalar guidance is often constrained by global trade-offs, yielding overly conservative interval boundaries. This once again corroborates the exceptional performance of the variable-sensitive dynamic intervention mechanism in enhancing the prediction quality of multivariate joint distributions.

\begin{figure}[p]
    \centering
    \begin{subfigure}[b]{1\textwidth}
        \centering
        \includegraphics[width=\linewidth]{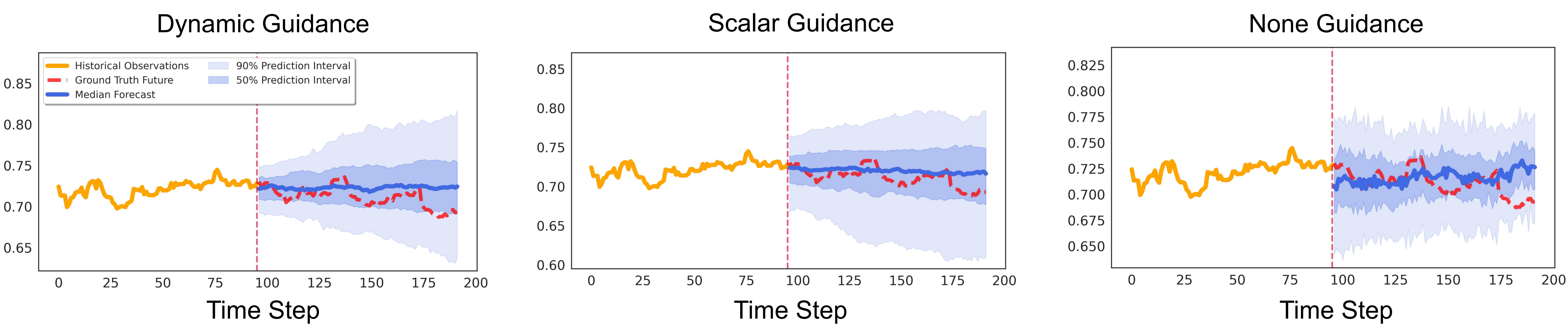}
        \caption{Exchange 7th dimension}
    \end{subfigure}
    
    \vspace{2em}
    
    \begin{subfigure}[b]{1\textwidth}
        \centering
        \includegraphics[width=\linewidth]{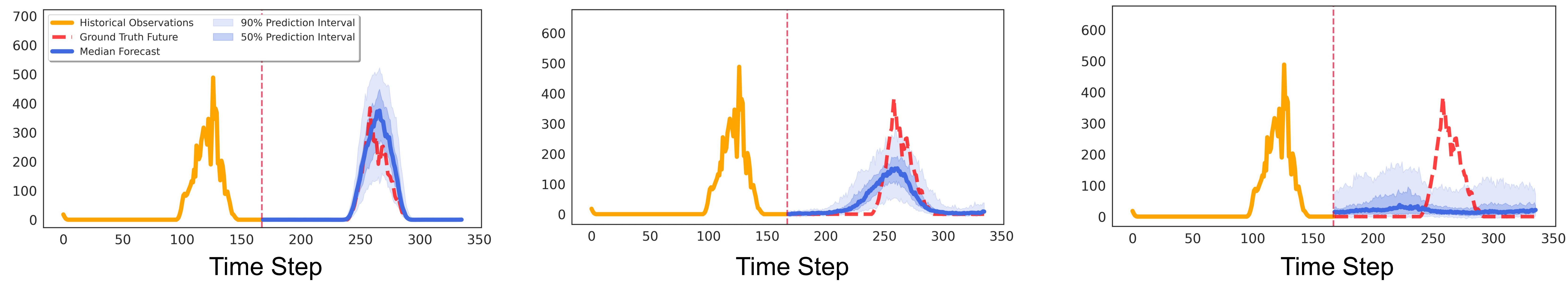}
        \caption{Weather 17th dimension}
    \end{subfigure}
    
    \vspace{2em}
    
    \begin{subfigure}[b]{1\textwidth}
        \centering
        \includegraphics[width=\linewidth]{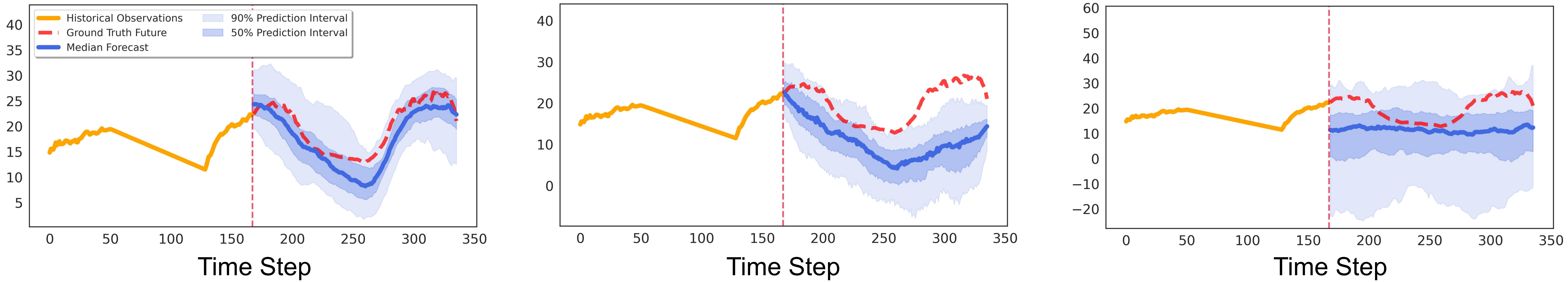}
        \caption{Appliance 12th dimension}
    \end{subfigure}
    
    \vspace{2em}
    
    \begin{subfigure}[b]{1\textwidth}
        \centering
        \includegraphics[width=\linewidth]{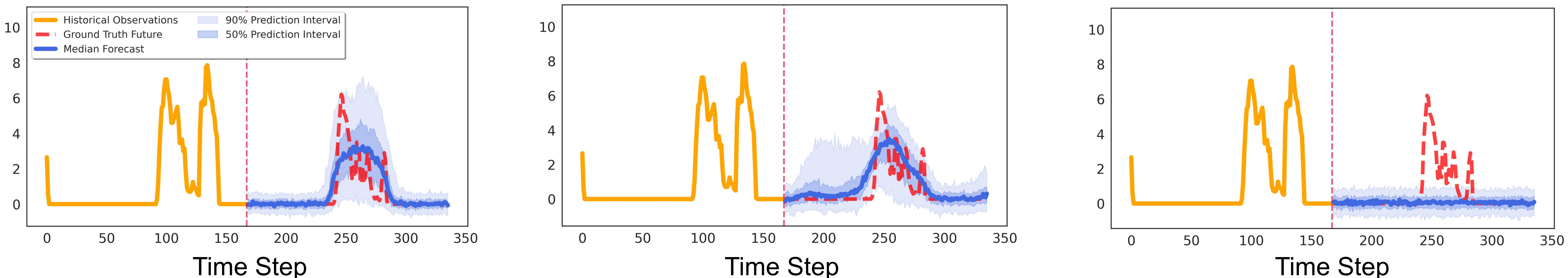}
        \caption{Solar 6th dimension}
    \end{subfigure}

    \vspace{2em}
    
    \begin{subfigure}[b]{1\textwidth}
        \centering
        \includegraphics[width=\linewidth]{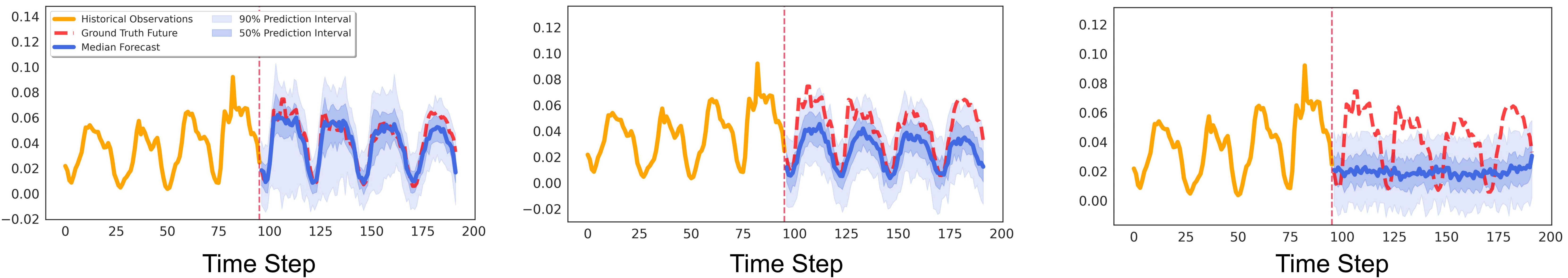}
        \caption{Traffic 3rd dimension}
    \end{subfigure}
    
    \caption{Comparison of probabilistic forecasting intervals generated by the model under different weight modes across various forecasting scenarios.}
    \label{fig:10}
\end{figure}

\clearpage

\subsection{Panoramic Views of Spatiotemporal Consistency in the Dynamic Weight Network}
\label{subapp:C.2}

In Section \ref{subsubsec:5.4.1}, using the Weather dataset as an example, we focused on analyzing the acute perception capabilities of the dynamic guidance network towards local abnormal features. Figure \ref{fig:11} further displays heat map comparisons for five datasets—ETTh1, Exchange, Appliance, Solar, and Traffic—at key time steps during the diffusion denoising process. In this experiment, we set $H=L=168$.

As observed from the experimental results, regardless of variations in variable dimensions and sampling frequencies across datasets, the dynamic guidance weights generated by the network consistently maintain a high degree of spatiotemporal semantic consistency with the true observational precision. For instance, in the Traffic heat map shown in Figure \ref{sub:e}, the true precision exhibits prominent "horizontal band-like" features, indicating the existence of long-term and inherent signal-to-noise ratio differences among different traffic nodes (variable channels). The dynamic weight network successfully captures this cross-channel spatial heterogeneity, continuously allocating strong guidance to high-confidence nodes while precisely isolating the interference from high-noise nodes. In contrast to Traffic, the true precision of the Solar dataset in Figure \ref{sub:d} presents regular "vertical block-like" features, which highly align with the common physical laws shared by all nodes in solar power generation scenarios (e.g., abrupt illumination changes caused by diurnal cycles). The network demonstrates robust temporal perception and adaptive capabilities towards this, adaptively and synchronously reducing the guidance strength across all channels during time periods when the overall system uncertainty surges.

\begin{figure}[p]
    \centering
    \begin{subfigure}[b]{0.8\textwidth}
        \centering
        \includegraphics[width=\linewidth, keepaspectratio]{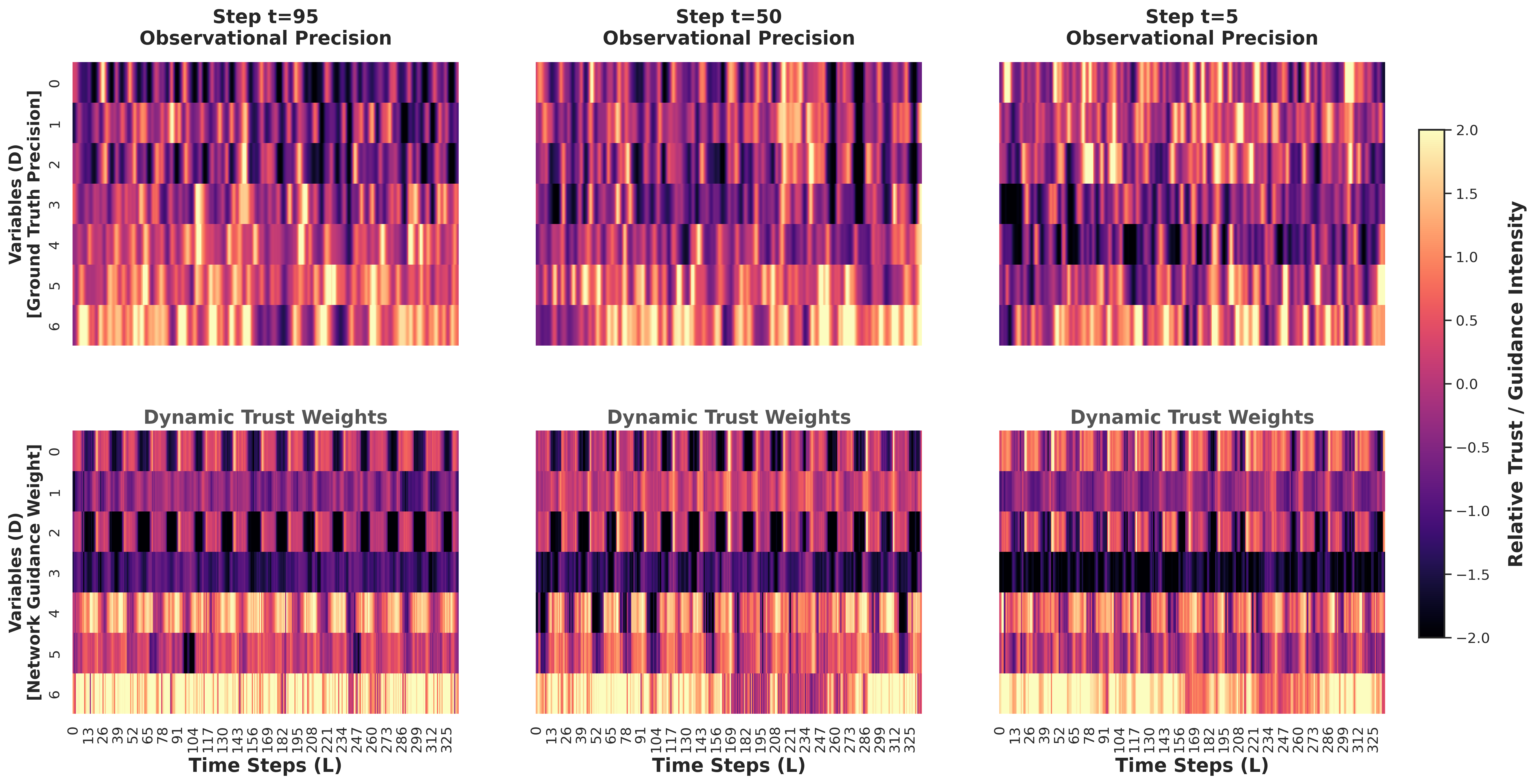}
        \caption{ETTh1}
        \label{sub:a}
    \end{subfigure}
    
    \vspace{1em}
    
    \begin{subfigure}[b]{0.8\textwidth}
        \centering
        \includegraphics[width=\linewidth, keepaspectratio]{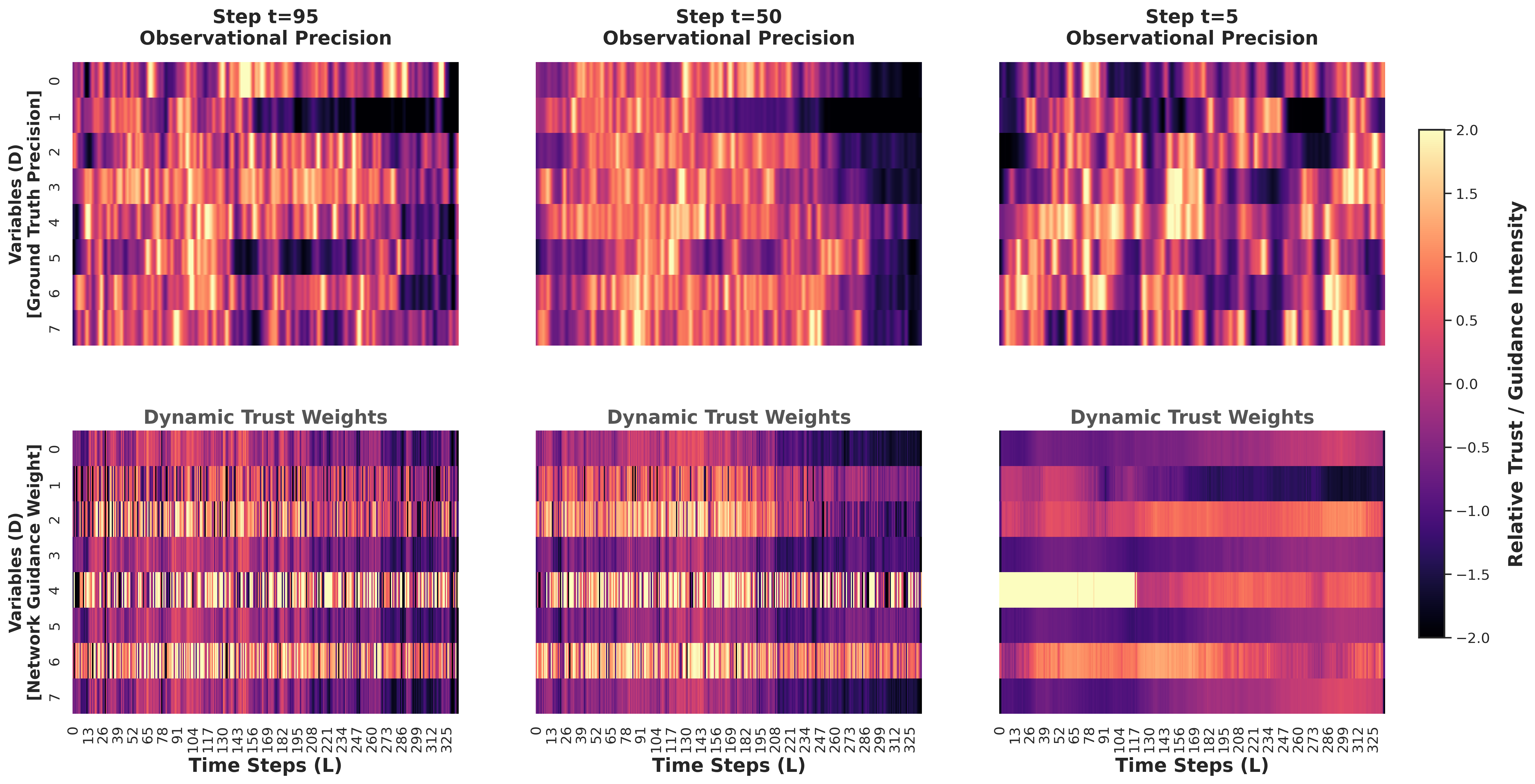}
        \caption{Exchange}
        \label{sub:b}
    \end{subfigure}
    
    \vspace{1em}
    
    \begin{subfigure}[b]{0.8\textwidth}
        \centering
        \includegraphics[width=\linewidth, keepaspectratio]{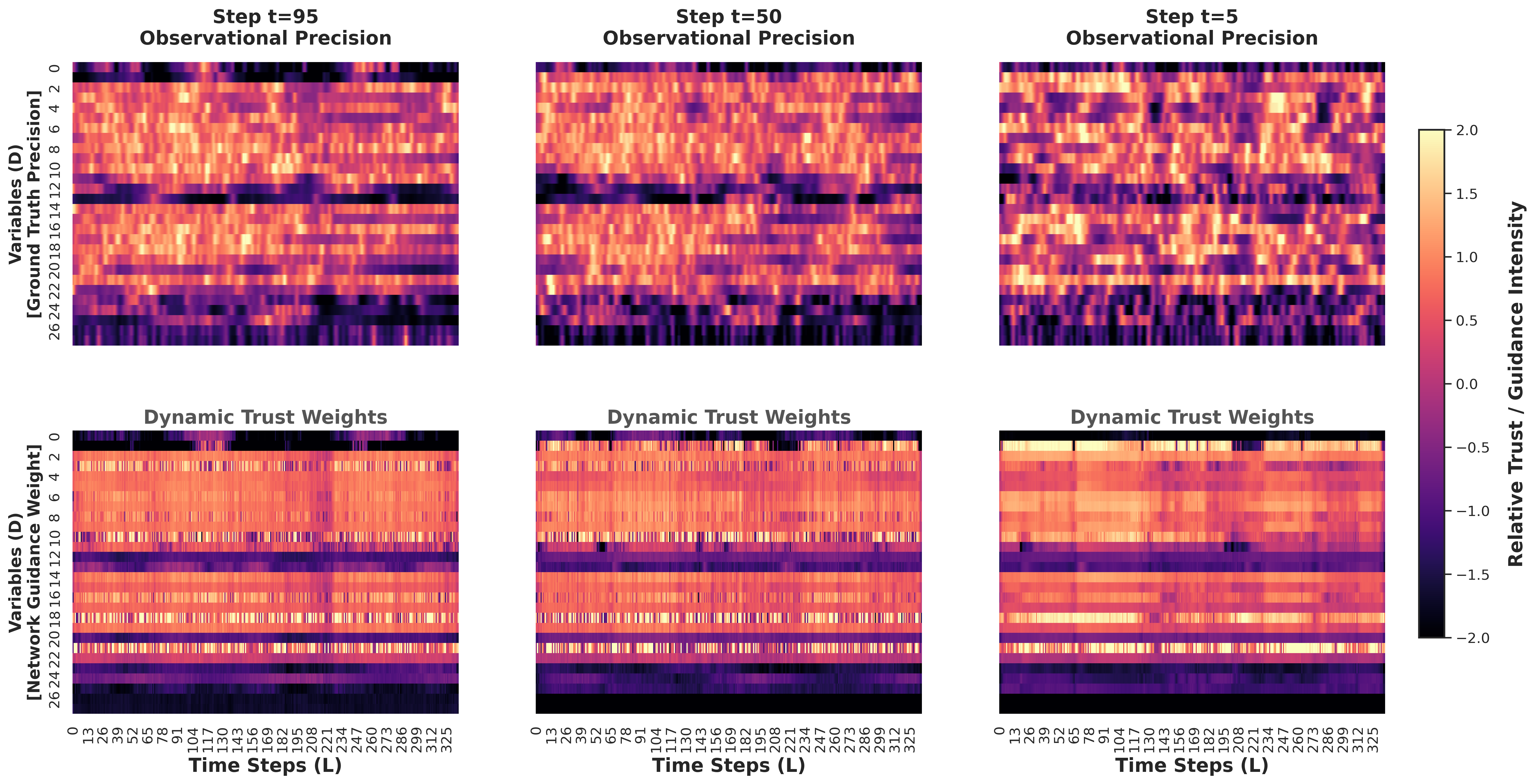}
        \caption{Appliance}
        \label{sub:c}
    \end{subfigure}
\end{figure}

\begin{figure}[p]
    \ContinuedFloat
    \centering
    \begin{subfigure}[b]{0.8\textwidth}
        \centering
        \includegraphics[width=\linewidth, keepaspectratio]{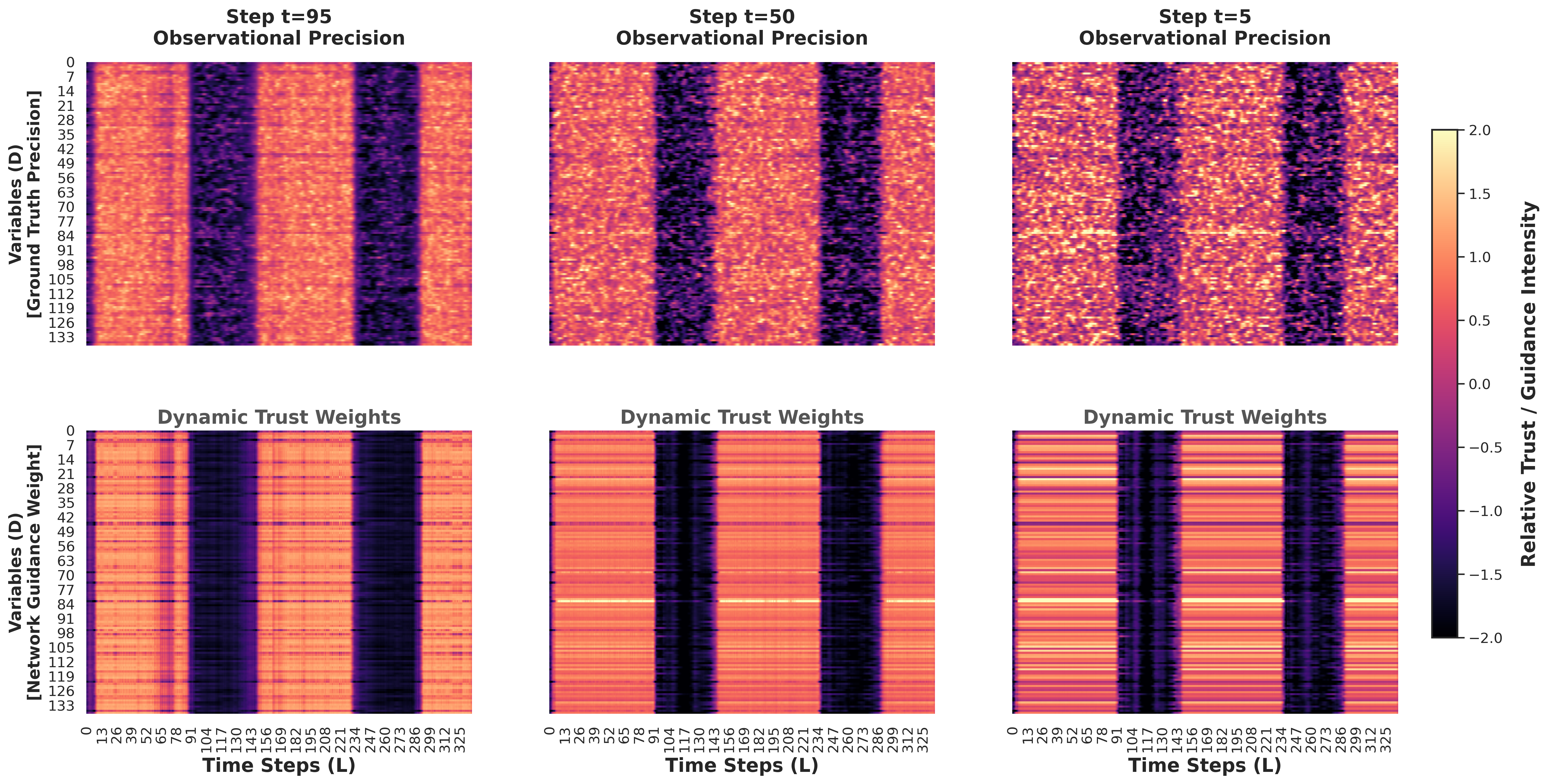}
        \caption{Solar}
        \label{sub:d}
    \end{subfigure}
    
    \vspace{1em}
    
    \begin{subfigure}[b]{0.8\textwidth}
        \centering
        \includegraphics[width=\linewidth, keepaspectratio]{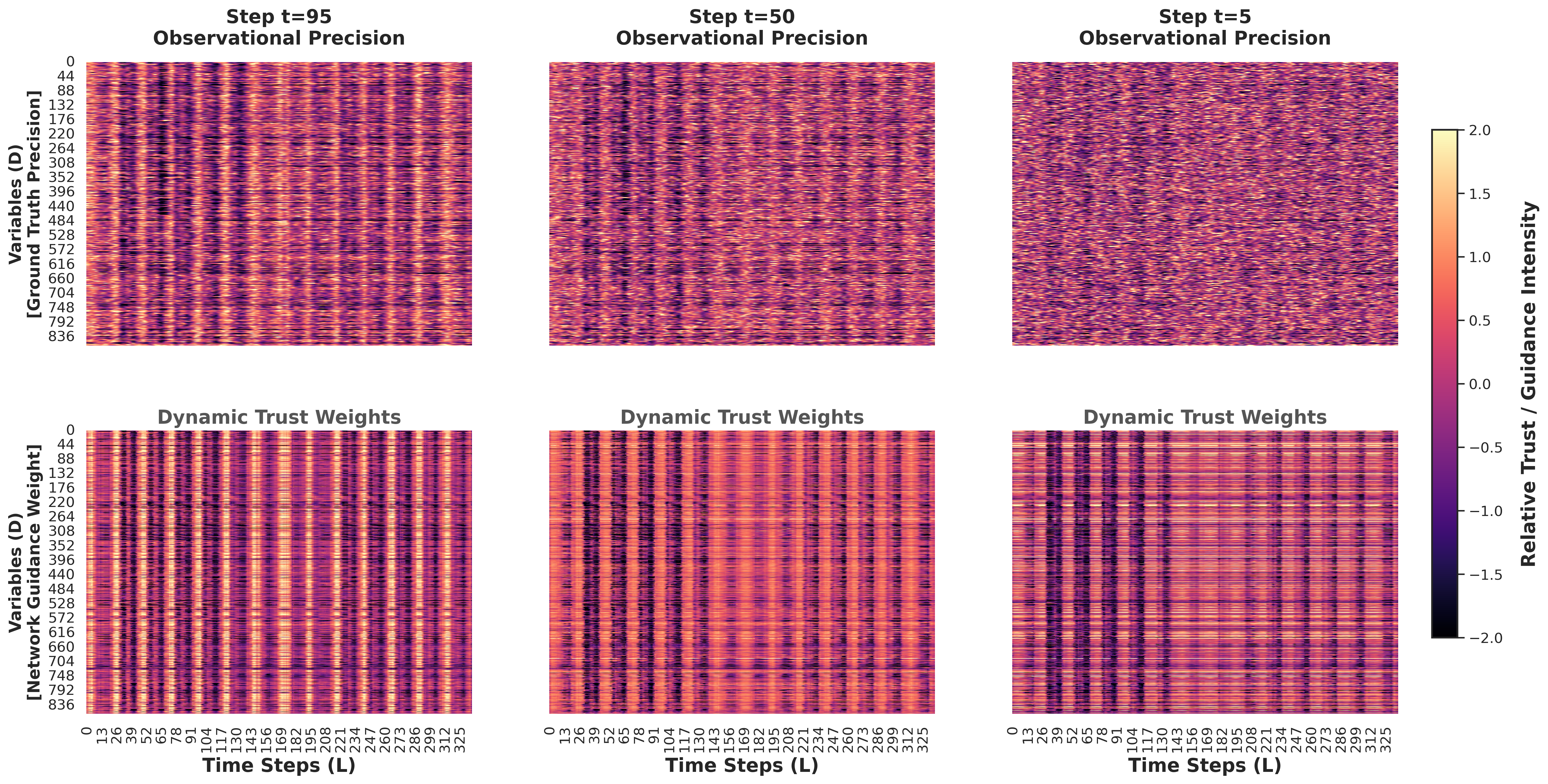}
        \caption{Traffic}
        \label{sub:e}
    \end{subfigure}
    
    \caption{Heatmap comparisons between true observational precision and network-generated dynamic guidance weights at key time steps of the diffusion process across different datasets.}
    \label{fig:11}
\end{figure}

\clearpage

\bibliographystyle{unsrt}

\bibliography{refs}

\end{document}